\documentclass[lettersize,journal]{IEEEtran}
\usepackage{amsmath,amsfonts}
\usepackage{algorithmic}
\usepackage{algorithm}
\usepackage{array}
\usepackage[caption=false,font=footnotesize]{subfig}
\usepackage{textcomp}
\usepackage{stfloats}
\usepackage{url}
\usepackage{verbatim}
\usepackage{graphicx}
\usepackage{cite}
\usepackage{multirow}
\usepackage{booktabs}
\usepackage{placeins}

\usepackage{amsthm}
\theoremstyle{plain}
\newtheorem{theorem}{Theorem}[section]
\newtheorem{corollary}[theorem]{Corollary}
\newtheorem{lemma}{Lemma}[section]

\newtheorem{proposition}[theorem]{Proposition}

\theoremstyle{definition}
\newtheorem{assumption}{Assumption}[section]

\theoremstyle{remark}
\newtheorem{remark}{Remark}[section]

\begin{document}
\title{SleepWalking: Privileged Representation Shaping for End-to-End Blind Locomotion in Legged Robots}

\author{
	Zheng Pan, Tenghui Wang, Peilin Li, Shiyu Zhou, Hao Sun, Yan Ma, Liang Yu, Liang He
	\thanks{Zheng Pan, Hao Sun, Yan Ma, Liang Yu, and Liang He are with Northwestern Polytechnical University, Xi'an 710072, China.}
	\thanks{Tenghui Wang and Shiyu Zhou are with Shanghai Jiao Tong University, Shanghai 200240, China.}
	\thanks{Peilin Li is with Yunmu Intelligent Manufacturing Co., Ltd., Jiangsu 215400, China.}
	\thanks{Co-corresponding authors are Liang Yu (liang.yu@nwpu.edu.cn) and Liang He (2021050018@nwpu.edu.cn)}
}

\maketitle

\begin{abstract}
Partially observable locomotion requires a policy to act when task-relevant properties of the robot--environment state are not fully specified by instantaneous observations. Existing approaches often address this challenge by explicitly estimating missing physical variables or processing extended observation histories through structured architectures.
We take a different view: partial observability is fundamentally an information-retention problem. The decisive question is not how task-relevant information enters the network, but whether the policy's internal state retains it. Guided by this perspective, we propose SleepWalking for Robot Locomotion (SWAQ), a one-stage end-to-end framework that uses next-step privileged physical reconstruction to shape what a recurrent history representation retains during policy learning, while the deployed actor uses only a direct history-to-action pathway.
Under aligned training settings, SWAQ achieves a 15.0\% higher peak mean terrain level than DWAQ, the strongest non-exteroceptive baseline, while using 44.4\% fewer inference MACs per control step. Layerwise probes further show that information associated with the reconstructed physical variables remains linearly decodable through the policy head up to the layer preceding the action output. Complementary theoretical analysis relates privileged-variable recoverability to the achievable-return gap between history-based and privileged-information policy classes.
These results suggest that semantic objectives can structure learning without requiring a corresponding architectural decomposition of the deployed controller.
\end{abstract}


\section{Introduction}

\IEEEPARstart{L}{egged}
robots can use discrete contacts to traverse obstacles and discontinuous terrain, making them suitable for operation in human-centered and unstructured environments. Model-based control established locomotion over rough terrain and
stairs~\cite{bellicoso2018,fankhauser2018,qi2021,grandia2023}, while deep reinforcement learning (RL) substantially expanded the agility and robustness
of legged locomotion~\cite{hwangbo2019,lee2019,tsounis2020,lee2020,rudin2022b}.

Terrain information is important for anticipating stairs, gaps, obstacles, and feasible footholds. Some learning-based controllers combining exteroception with proprioception have achieved agile locomotion in natural environments and dynamic parkour~\cite{acero2022,miki2022,duan2024,cheng2024a}. However, exteroceptive measurements may be noisy, incomplete, or unavailable, and they increase the sensing and processing requirements of the control system \cite{nahrendra2026}. Removing this sensing channel leaves task-relevant terrain properties unavailable before interaction and observable only indirectly through the robot's physical response.

Incomplete access to task-relevant state is naturally modeled as partial observability, which arises broadly in robot localization, navigation, manipulation, tracking, and interaction under uncertain sensing and dynamics~\cite{kaelbling1998a,kurniawati2022,lauri2023a}. Different underlying robot-environment states can produce similar instantaneous observations while requiring different actions. A controller must therefore use its observation and action history to distinguish latent situations that cannot be resolved at a single time step. For deep RL, task return provides only indirect supervision for the learned representation; successful behavior alone does not reveal whether the policy has retained information governing future interactions or merely exploited correlations encountered during training~\cite{stooke2021}.

To study this mechanism, we choose blind locomotion as a controlled setting in which terrain observations are completely removed from the policy. Although terrain geometry is not directly measured, physical interaction produces proprioceptive evidence that can accumulate over time. Human stair negotiation provides a useful analogy: short-term memory of step geometry can support subsequent stepping after the initial interaction~\cite{graci2017}. History-conditioned policies have similarly enabled quadrupedal locomotion over challenging terrain, bipedal stair traversal, and one-stage blind locomotion~\cite{lee2020,siekmann2021c,aswinnahrendra2023}. Together, these studies show that interaction history can support locomotion without direct terrain observations.
Modular pipelines facilitate separate supervision and optimization of individual components, but a predefined intermediate estimate may also restrict the information available to downstream control. This matters under partial observability because noisy or incomplete observations can remain informative when their reliability is inferred from context~\cite{miki2022,radosavovic2024a}.

The central question is therefore whether missing physical variables must be explicitly estimated and supplied to the controller, or whether making them recoverable from its internal history representation during training is sufficient to shape a useful decision state.
To explore the latter possibility, we propose \textbf{SWAQ} (\textbf{S}leep\textbf{WA}lking for Robot Locomotion), a one-stage end-to-end asymmetric actor-critic framework. The name is inspired by somnambulism, a non-rapid-eye-movement parasomnia in which complex motor behavior can proceed under reduced awareness and sleepwalkers may navigate familiar surroundings~\cite{zadra2013}. Guided by this analogy, SWAQ jointly learns the control policy and a recurrent representation of the nonprivileged action-observation history, while state-and-terrain reconstruction shapes this representation during training. This auxiliary objective makes physical and terrain variables recoverable from the shared representation without routing reconstructed quantities into the controller. In this setting, SWAQ instantiates a broader architectural principle: semantic task decomposition can be imposed through training objectives rather than mandatory runtime interfaces.

The main contributions of this work are summarized as follows:

\begin{enumerate}
	\item  We formulate and study an information-retention view of history-based robot control under partial observability: the central design object is the task-relevant information retained in the policy's internal state, rather than the architectural route by which that information is carried. Under this view, semantic supervision can structure learning without prescribing a corresponding architectural decomposition at deployment.
	
	\item We develop SWAQ and provide supporting evidence for this view at both the representation and policy levels. Selected robot-state and terrain information remains linearly decodable near the action output, while aligned comparisons show improved terrain progression and lower inference cost than the strongest non-exteroceptive baseline.
	
	\item We theoretically analyze SWAQ's privileged reconstruction objective by connecting reconstruction accuracy to representation-level information retention and, under explicit assumptions, characterizing privileged-policy emulation and the corresponding achievable-return gap.
	
\end{enumerate}

\section{Related Work}

\subsection{Proprioception-Based Blind Locomotion}

Blind locomotion seeks to traverse uneven terrain without exteroceptive measurements or an explicit terrain map. Since local geometry and contact conditions are not directly observable, policies must extract useful cues from proprioceptive feedback and the resulting robot-environment
interactions. In \cite{lee2020}, a temporal convolutional network was trained through privileged teacher-student learning to infer terrain-relevant information from a history of proprioceptive observations, enabling a quadruped to traverse challenging natural terrain. In \cite{siekmann2021c}, stair-like terrain randomization was incorporated into sim-to-real reinforcement learning, allowing a bipedal robot to traverse real-world stairs using only proprioceptive feedback and without an explicit terrain model. These results establish that interaction histories can support robust locomotion even when the terrain is not directly perceived.

Subsequent work has expanded the range of blind locomotion skills while introducing different mechanisms for handling terrain-dependent behavior. In \cite{chamorro2024}, a shared policy was conditioned on a binary terrain indicator that activates or deactivates a stair-climbing mode. This conditioning enables distinct behaviors for stairs and regular terrain, but requires the terrain mode to be specified during deployment rather than inferred entirely from proprioceptive history. In DreamWaQ (DWAQ)~\cite{aswinnahrendra2023}, the teacher-student procedure was replaced by a one-stage asymmetric actor-critic framework equipped with a context-aided estimator. The estimator extracts a latent terrain context from proprioceptive history and estimates the robot velocity, both of which are supplied to the control policy during deployment. These methods condition deployed control on either an externally specified terrain mode or an estimated latent variable.

Privileged teacher-student methods first optimize a policy with access to simulator-only information and subsequently train a deployable student or adaptation module~\cite{lee2020,kumar2021a}. This sequential procedure introduces additional training stages, while behavioral cloning or latent regression may restrict the student to information and behaviors represented by the teacher~\cite{aswinnahrendra2023,luo2024}. Distribution shift may also arise because the deployed student visits states not adequately covered by teacher-generated trajectories.

One-stage methods alleviate this training separation, but estimator-conditioned architectures may still place an explicit intermediate interface in the control pathway at deployment. Errors or information discarded by this interface can propagate to the downstream policy.

\subsection{History and Auxiliary Representation Learning for Control}

Reinforcement-learning objectives provide task-directed supervision, but the reward signal alone may be insufficient to induce a representation that preserves the latent dynamics and environmental factors relevant to control. Auxiliary objectives provide representation-level supervision beyond the task reward and have been used to shape shared policy encoders~\cite{jaderberg}. Temporally structured self-supervision can further produce representations suitable for downstream control \cite{stooke2021}. These results motivate auxiliary representation learning, but do not determine whether recovered physical variables should become explicit inputs to a deployed locomotion policy.

In legged locomotion, history-dependent representations are commonly used to recover physical quantities that are unavailable from an instantaneous observation. In \cite{kumar2021a}, an adaptation module maps a recent state-action history to a latent environment vector that is consumed by the base policy during deployment. 
In \cite{aswinnahrendra2023}, a context-aided estimator is trained through velocity estimation and next-observation reconstruction, and supplies the resulting velocity estimate and latent terrain context to the deployed locomotion policy. These approaches explicitly expose the estimated latent variables to the control pathway. A different design is adopted in the concurrent perceptive-locomotion work \cite{chen}, where an auxiliary decoder reconstructs a privileged terrain height scan from the hidden state of a recurrent depth policy. The decoder is removed after training, allowing the recurrent policy to retain terrain information without explicitly reconstructing the height scan during deployment. However, the policy is obtained through a multi-stage pipeline involving privileged teachers, multi-skill distillation, and subsequent fine-tuning, and its deployed observation includes exteroceptive depth measurements. These methods differ in whether latent estimates are consumed at deployment, whether exteroception is required, and whether policy and representation learning are performed in a single optimization stage. Within this design space, SWAQ jointly optimizes the policy and history representation in one stage, using privileged reconstruction targets to supervise the representation without feeding reconstructed quantities to the proprioception-only deployed actor.

Attention provides a complementary mechanism for selecting task-relevant information. He et al.~\cite{he2025b} use a proprioceptive embedding as the query and point-wise terrain features as keys and values, allowing the map encoder to attend to state-dependent terrain regions associated with future footholds. Radosavovic et al.~\cite{radosavovic2024a} instead apply causal self-attention to observation-action tokens to select temporal evidence relevant to gait, contact, and in-context adaptation. These approaches perform spatial and temporal information selection, respectively, complementing auxiliary objectives that explicitly shape the information retained in learned representations.

At a larger scale, recent world action models jointly model future visual observations and actions to acquire action-conditioned representations of physical dynamics \cite{ye2026}. Notably, \cite{yuan2026} reports that video co-training can retain much of its control benefit even when explicit future generation is removed at inference, suggesting that predictive objectives may be valuable primarily for shaping the latent representation during training.

\section{Method}
\label{section:Method}

\subsection{Problem Formulation}

\subsubsection{Notation}

Table~\ref{table:Notation} summarizes the core notation used across the method and theoretical analysis. Symbols that are local to a particular derivation are defined at their first occurrence.

\begin{table}[htbp]
	\centering
	\caption{Core Notation Used Throughout This Paper.}
	\label{table:Notation}
	\renewcommand{\arraystretch}{1.08}
	\begin{tabular}{
			@{}>{\raggedright\arraybackslash}p{0.2\columnwidth}>{\raggedright\arraybackslash}p{0.75\columnwidth}@{}
		}
		\toprule
		\textbf{Symbol}  & \textbf{Definition} \\
		
		\midrule
		$\mathcal P$& Original POMDP. \\
		
		$\mathcal S,\mathcal A,\mathcal O$
		& Latent-state, action, and policy-observation spaces, respectively. \\
		
		$s_t,\boldsymbol a_t,\boldsymbol o_t$
		& Latent robot-environment state, control action, and policy observation at time $t$, respectively. \\
		
		$P,\Omega$
		& Latent-state transition kernel and observation kernel of $\mathcal P$, respectively. \\

		$r$; $r_{\max}$
		& One-step reward and its uniform absolute bound used in the theoretical analysis, respectively. \\
		
		$\tau_t\in\mathcal T$
		& Nonprivileged action-observation history available to the deployed policy at time $t$. \\
		
		$\mathcal M_{\mathcal T}$
		& Exact MDP induced by taking the complete nonprivileged history as the state. \\
		
		$\overline P_{\mathcal T},
		\overline r_{\mathcal T}$
		& Transition kernel and expected one-step reward of $\mathcal M_{\mathcal T}$, respectively. \\
		
		$\boldsymbol H_t\in\mathcal H$; $\boldsymbol h_t$
		& Random shared history representation produced by the recurrent encoder and shared trunk, and its realized value at time $t$, respectively. \\
		
		$X_k\in\mathcal X$
		& Common augmented pre-action Markov state $(s_k,\tau_k,\boldsymbol H_{k-1},\boldsymbol Y_k)$ used only for the theoretical analysis. \\
		
		$\boldsymbol O_k\in\mathcal O$
		& Current nonprivileged policy observation at decision epoch $k$; it excludes the preceding history representation $\boldsymbol H_{k-1}$. \\
		
		$\boldsymbol Y_k\in\mathcal Y$; $\widehat{\boldsymbol Y}_k$
		& Privileged target available at decision epoch $k$ and its reconstruction from $\boldsymbol H_{k-1}$, respectively, where $k=t+1$ for the transition indexed by $t$. \\
		
		$g_{\theta_{\mathrm{rec}}}$; $g$
		& Implemented reconstruction mapping and its parameter-suppressed theoretical counterpart, respectively. \\

		$\boldsymbol W$
		& Fixed diagonal normalization matrix defining the weighted reconstruction distortion. \\

		$\mu$
		& Population reconstruction-training distribution over the temporally aligned tuple $(\boldsymbol O,\boldsymbol H,\boldsymbol Y)$. \\
		
		$\Gamma$
		& Measurable projection from $X_k\in\mathcal X$ to the temporally aligned tuple $(\boldsymbol O_k,\boldsymbol H_{k-1},\boldsymbol Y_k)$. \\
		
		$\gamma$; $J(\pi)$; $J^{*}(\Pi)$; $J_{\mathcal I}^{*}$
		& Discount factor, expected discounted return of policy $\pi$, optimal achievable return over policy class $\Pi$, and its shorthand for policies conditioned on information set $\mathcal I$, respectively. \\
		\bottomrule
	\end{tabular}

\end{table}

\subsubsection{POMDP and History-State Formulation}
\label{sec:pomdp_history_formulation}

Blind legged locomotion is naturally partially observable because the controller must make decisions without direct exteroceptive measurements of the surrounding terrain. We therefore formulate the control problem as a partially observable Markov decision process (POMDP), a standard framework for robot decision making under incomplete and uncertain state information~\cite{lauri2023a}. The POMDP is defined as
\begin{equation}
	\mathcal P
	=
	\left\langle
	\mathcal S,
	\mathcal A,
	\mathcal O,
	P,
	\Omega,
	r,
	\rho_0,
	\gamma
	\right\rangle,
	\label{eq:pomdp_definition}
\end{equation}
where $\mathcal S$, $\mathcal A$, and $\mathcal O$ denote the latent state, action, and observation spaces, respectively; $P(\mathrm d s'\mid s,\boldsymbol a)$ is the state-transition kernel; $\Omega(\mathrm d\boldsymbol o'\mid s',\boldsymbol a)$ is the observation kernel; $r(s,\boldsymbol a,s')$ is the one-step reward; $\rho_0$ is the initial-state distribution; and $\gamma\in[0,1)$ is the discount factor.

In blind locomotion, the latent state $s_t$ contains both the complete robot state and task-relevant environmental variables. By contrast, the policy receives only the observation $\boldsymbol o_t$, which is primarily composed of proprioceptive measurements and does not directly reveal the local terrain geometry or other latent robot-environment interaction conditions. Consequently, different latent states may produce similar instantaneous observations. For example, similar joint configurations can occur immediately before contact with flat ground, a stair edge, or an irregular foothold, although these situations require different control responses. Thus, the instantaneous observation $\boldsymbol o_t$ is generally not a Markov state for the deployed policy.

Define the complete nonprivileged action-observation history as
\begin{equation}
	\tau_t
	:=
	(\boldsymbol o_0,\boldsymbol a_0,\boldsymbol o_1,\ldots,\boldsymbol a_{t-1},\boldsymbol o_t)
	\in\mathcal T.
	\label{eq:nonprivileged_history}
\end{equation}

Following the history-state construction for POMDPs~\cite{nguyen2021}, taking $\tau_t$ as the state induces the MDP
\begin{equation}
	\mathcal M_{\mathcal T}
	=
	\left\langle
	\mathcal T,
	\mathcal A,
	\overline P_{\mathcal T},
	\overline r_{\mathcal T},
	\rho_{\mathcal T,0},
	\gamma
	\right\rangle,
	\label{eq:history_induced_mdp}
\end{equation}

For every measurable $B\subseteq\mathcal T$, its transition kernel and expected one-step reward are
\begin{align*}
	\overline P_{\mathcal T}(B\mid\tau,\boldsymbol a)
	&:=
	\Pr(\tau_{t+1}\in B\mid\tau_t=\tau,\boldsymbol a_t=\boldsymbol a),\\
	\overline r_{\mathcal T}(\tau,\boldsymbol a)
	&:=
	\mathbb E[r_t\mid\tau_t=\tau,\boldsymbol a_t=\boldsymbol a],
\end{align*}
and $\rho_{\mathcal T, 0}$ is the initial distribution over histories.

Let $\Pi_{\mathrm{hist}}$ denote the class of admissible history-dependent policies, where $\boldsymbol a_t\sim\pi(\cdot\mid\tau_t)$ for $\pi\in\Pi_{\mathrm{hist}}$. The discounted return of a policy $\pi$ is defined as
\begin{equation}
	J(\pi)
	:=
	\mathbb E_{\pi}
	\left[
	\sum_{t=0}^{\infty}
	\gamma^t r_t
	\right],
	\label{eq:discounted_return}
\end{equation}
where $r_t:=r(s_t,\boldsymbol a_t,s_{t+1})$ and the expectation is taken over the trajectories induced by the initial-state distribution, transition kernel, observation kernel, and policy. For any admissible policy class $\Pi$, define its optimal achievable return as
\begin{equation}
	J^{*}(\Pi)
	:=
	\sup_{\pi\in\Pi}J(\pi).
	\label{eq:optimal_achievable_return}
\end{equation}

The optimization objective is to find a history-dependent policy that maximizes the expected discounted return:
\begin{equation}
	\pi^\star
	\in
	\underset{\pi\in\Pi_{\mathrm{hist}}}{\arg\max}
	\;J(\pi).
	\label{eq:optimal_history_policy}
\end{equation}

The remaining challenge is therefore not only to compress the growing history, but to determine which information a finite-dimensional representation should preserve for control without introducing additional deployment observations.

\subsubsection{Action and Observation Spaces}

For action space, the policy action $\boldsymbol a_t\in\mathbb R^{n_j}$ specifies the relative targets for the $n_j$ actuated joints. The applied torque of each actuated joint is computed by a simple PD controller.

For observation space, the actor receives only proprioceptive observation $\boldsymbol o_t$,  comprising the base angular velocity $\boldsymbol{\omega}^{b}_t\in\mathbb R^3$, projected gravity $\boldsymbol g^{b}_t\in\mathbb R^3$, velocity command $\boldsymbol c_t\in\mathbb R^3$, relative joint positions $\boldsymbol q_t-\boldsymbol q^{0}\in\mathbb R^{n_j}$, joint velocities $\dot{\boldsymbol q}_t\in\mathbb R^{n_j}$, previous action $\boldsymbol a_{t-1}\in\mathbb R^{n_j}$, and a two-dimensional gait-phase signal $\boldsymbol{\rho}_t\in\mathbb R^2$.

\subsubsection{Reward Design}

Because this work focuses on history-representation learning rather than reward engineering, we adopt an established locomotion reward structure and apply it unchanged to SWAQ and all ablation variants.
Following the phase-based contact formulation in \cite{margolis2023}, the gait terms penalize vertical contact forces during swing and tangential foot velocities during stance, thereby encouraging the prescribed contact schedule. The remaining task-tracking and motion-regularization terms follow commonly used formulations in learning-based locomotion \cite{margolis2023,aswinnahrendra2023,sun2025a}. Table~\ref{table:reward_setting} summarizes all reward terms and their weights.

\begin{table}[htbp]
	\centering
	\caption{Reward Structure.}
	\label{table:reward_setting}
	\renewcommand{\arraystretch}{1.08}
	\begin{tabular}{
			@{}>{\raggedright\arraybackslash}p{0.18\columnwidth}
			>{\raggedright\arraybackslash}p{0.58\columnwidth}
			>{\raggedleft\arraybackslash}p{0.1\columnwidth}@{}
		}
		\toprule
		\textbf{Category} & \textbf{Reward Term} & \textbf{Weight} \\
		\midrule
		
		\multirow{2}{*}{\textbf{Task}} 
		& Track linear velocity (XY, Exp) & 5 \\
		& Track angular velocity (Z, Exp) & 5 \\
		
		\midrule
		\multirow{2}{*}{\textbf{Gait}} 
		& Feet swing phase force (Z, Exp)  & -5 \\
		& Feet stance phase velocity (XY, Exp)  & -5 \\
		
		\midrule
		
		\multirow{2}{*}{\textbf{Regulation}} 
		& Is terminated (Bool)& -200 \\
		& Body angular velocity (XY, L2)   & -0.05 \\
		& Body linear velocity (Z, L2)  & -1 \\
		& Body flat orientation (L2)  & -5 \\
		& Joints torque (L2) & -1e-5 \\
		& Joints acceleration (L2) & -1e-6 \\
		& Joints power (L1) & -1e-5 \\
		& Joints position limit (90\%, L1)  & -1 \\
		& Joints action rate (L2) & -0.01 \\
		& Feet stumble (Bool) & -2.5 \\
		& Feet slide (L2) & -0.25 \\
		& Feet distance (Y, Exp)  & -100 \\
		& Feet force too large (Z, L1)  & -2e-3 \\
		& Feet flat orientation (L2) & -2.5 \\
		& Feet height over terrain ($\ge$0.1m, L2)  & -200 \\

		\bottomrule
		
\end{tabular}
\end{table}

\subsection{SWAQ Architecture}

As depicted in Fig.~\ref{fig:Overview_of_SWAQ}, SWAQ is a one-stage end-to-end asymmetric actor-critic framework. Its actor contains a recurrent history encoder and a shared trunk followed by a policy head, which together form the deployed control pathway. A training-only reconstruction pathway branches from the shared representation and predicts privileged physical targets aligned with the next control step. Its outputs are used only to supervise the shared representation and are not routed into the policy head.

\begin{figure*}[t]
	\centering
	\includegraphics[width=7 in]{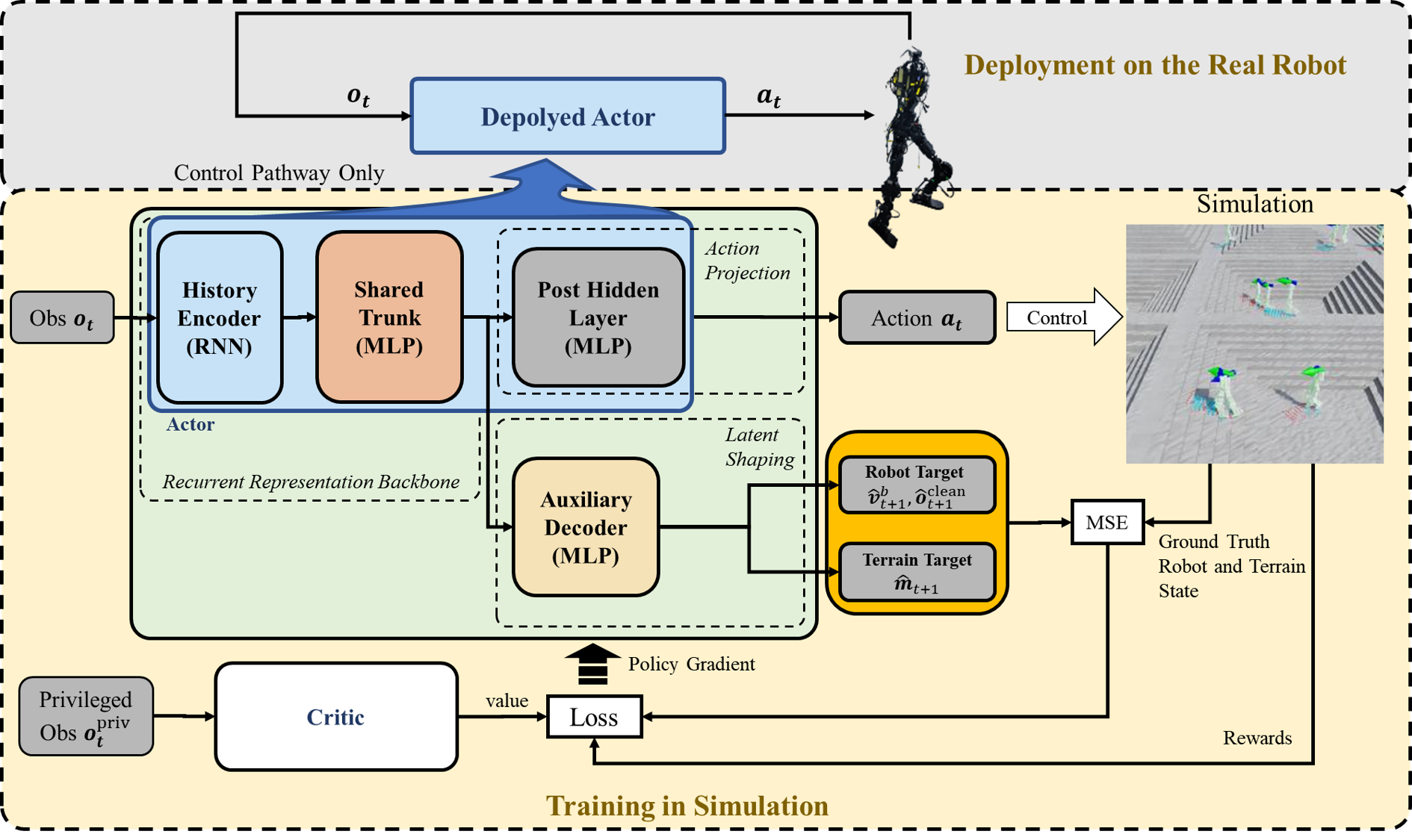}
	\caption{
		Overview of the proposed SWAQ framework. At step $t$, the actor produces $\boldsymbol a_t$, and the resulting simulator transition provides the next-step reconstruction target $\boldsymbol y_{t+1}$. Training-only next-step reconstruction shapes the shared history representation during policy learning, whereas the deployed actor maps nonprivileged history directly to actions without consuming any reconstructed quantity.
	}
	\label{fig:Overview_of_SWAQ}
\end{figure*}

The recurrent encoder and shared trunk produce
$\boldsymbol h_t=f_{\theta_{\mathrm{enc}}}(\tau_t)$, which is consumed by both pathways. The control pathway maps $\boldsymbol h_t$ to $\boldsymbol a_t$, whereas the training-only reconstruction pathway maps the same representation to the reconstructed target $\widehat{\boldsymbol y}_{t+1}=g_{\theta_{\mathrm{rec}}}(\boldsymbol h_t)$. Because both mappings depend on $\boldsymbol h_t$, the reconstruction loss backpropagates through the shared representation backbone.

\subsubsection{Recurrent Asymmetric Actor-Critic}

Although this reformulation is exact, the dimension of the complete history grows with time. Classical POMDP methods address this issue using a Bayesian belief state, which is a sufficient statistic of the action-observation history~\cite{kaelbling1998a}. 
SWAQ instead uses an LSTM recurrent encoder~\cite{hausknecht,heess2015,peng2018b} to recursively compress $\tau_t$ into the finite-dimensional history representation
$\boldsymbol h_t$.

\subsubsection{Training-Only Auxiliary Representation Learning}

Policy optimization alone does not explicitly determine which physical information should remain recoverable from the finite-dimensional history representation. SWAQ therefore supervises $\boldsymbol h_t$ using selected robot-state and terrain variables available in simulation. The targets are aligned with the next control step, introducing a predictive constraint while assigning the learned representation explicit physical semantics.

The training-only reconstruction pathway is supervised by simulator variables aligned with the next control step. Its targets comprise the base linear velocity $\boldsymbol v^b_{t+1}\in\mathbb R^3$, noise-free policy observation $\boldsymbol{o}^{\mathrm{clean}}_{t+1}\in \mathcal O$, and privileged terrain height scan $\boldsymbol m_{t+1}\in\mathbb R^{\text{length}\times \text{width}}$. Their concatenation defines the random reconstruction target
\begin{equation}
	\boldsymbol Y_{t+1}
	=
	\left[
	(\boldsymbol V^b_{t+1})^\top,
	(\boldsymbol O^{\mathrm{clean}}_{t+1})^\top,
	\operatorname{vec}(\boldsymbol M_{t+1})^\top
	\right]^\top .
	\label{eq:reconstruction_target}
\end{equation}

At step $t$, the control pathway produces $\boldsymbol a_t$ from $\boldsymbol h_t$, and the ensuing simulator transition yields a sample $\boldsymbol y_{t+1}$ of $\boldsymbol Y_{t+1}$. The action-history component of $\boldsymbol o^{\mathrm{clean}}_{t+1}$ is therefore the realized action $\boldsymbol a_t$. Equivalently, at decision epoch $k=t+1$, it is the previous action $\boldsymbol a_{k-1}$ and is available before $\boldsymbol a_k$ is selected.
The reconstruction pathway observes this completed transition only during training. Lowercase $\boldsymbol y_{t+1}$ and $\widehat{\boldsymbol y}_{t+1}$ denote the ground truth target and reconstruction samples used in the loss below.

\paragraph{State-and-terrain reconstruction}

The reconstruction targets are generated automatically from simulator rollouts and require no human annotations. With a diagonal normalization matrix $\boldsymbol W$, the reconstruction loss is

\begin{equation}
	\mathcal L_{\mathrm{rec}}
	=
	\left\|
	\boldsymbol W\left(\boldsymbol y_{t+1}-\widehat{\boldsymbol y}_{t+1}\right)
	\right\|_2^2.
	\label{eq:reconstruction_loss}
\end{equation}

Using targets aligned with step $t+1$ gives the reconstruction objective a predictive role. Reconstructing contemporaneous quantities would primarily encourage their decodability from the current history representation. In contrast, recovering $\boldsymbol Y_{t+1}$ requires $\boldsymbol h_t$ to retain information that is predictive of the ensuing robot-environment transition. The robot-state components provide physical grounding, while the terrain target encourages the recurrent encoder to aggregate terrain evidence revealed through interaction, despite the absence of direct terrain measurements from the policy observation.

The reconstruction objective also penalizes complete representation collapse. If $\boldsymbol h_t$ is constant, a deterministic decoder can produce only a constant prediction, whose minimum expected squared error is given by the weighted variance of $\boldsymbol Y_{t+1}$. Therefore, achieving a reconstruction error below this constant-predictor baseline requires $\boldsymbol h_t$ to retain information predictive of the time-varying targets.

\paragraph{Design rationale}

Latent self-prediction provides a common mechanism for imposing temporal structure on a learned representation. Given a projection $\boldsymbol z_t=q(\boldsymbol h_t)$, a representative objective predicts its action-conditioned successor:
\begin{equation}
	\mathcal L_{\mathrm{latent}}
	=
	\left\|
	p\left( \boldsymbol z_t,\boldsymbol a_t \right)
	-
	\operatorname{sg}\!\left(q(\boldsymbol h_{t+1})\right)
	\right\|_2^2 ,
	\label{eq:latent_self_prediction_reference}
\end{equation}
where, $q$ and $p$ denote a projection and a latent predictor, respectively, and $\operatorname{sg}(\cdot)$ stops the gradient through the prediction target~\cite{schwarzer2020,ni2024}.

SWAQ does not optimize \eqref{eq:latent_self_prediction_reference}. Instead, it transfers the one-step predictive principle from an unconstrained latent successor to selected privileged physical outcomes. The reconstruction mapping $\boldsymbol h_t\mapsto\widehat{\boldsymbol y}_{t+1}$ therefore combines temporal prediction with explicit robot-state and terrain semantics. This target choice specifies what should remain recoverable from the history representation without making the reconstructed quantities inputs to the deployed policy.

Applying \eqref{eq:optimal_achievable_return} to the fully observed, complete-history, and observation-only policy classes gives the standard POMDP information hierarchy
\begin{equation*}
	J_{\mathcal S}^{*}
	\ge
	J_{\mathrm{hist}}^{*}
	\ge
	J_{\boldsymbol O}^{*}.
	\label{eq:standard_pomdp_return_order}
\end{equation*}

The inequalities follow from Markov-state sufficiency and the inclusion of observation-only policies in the complete-history policy class.

SWAQ uses the privileged target $\boldsymbol Y$ only during training to specify what a finite-dimensional history representation $\boldsymbol H$ should retain, without making $\boldsymbol Y$ an input to the deployed policy.
Combining Lemma~\ref{lem:baseline_information_ordering} with Corollary~\ref{corollary:exact_recovery_no_loss}, under their stated Markov-sufficiency, exact-recovery, and policy-class expressivity conditions, yields
\begin{equation}
	J_{\mathcal S}^{*}
	\ge
	J_{\mathrm{hist}}^{*}
	\ge
	J_{\boldsymbol O,\boldsymbol H}^{*}
	\ge
	J_{\boldsymbol O,\boldsymbol Y}^{*}
	\ge
	J_{\boldsymbol O}^{*}.
	\label{eq:exact_recovery_return_order}
\end{equation}

For approximate recovery,
Corollary~\ref{corollary:achievable_return_bound} bounds
$J_{\boldsymbol O,\boldsymbol Y}^{*,L,C}
-J_{\boldsymbol O,\boldsymbol H}^{*}$ by
$\mathcal O(\sqrt{\varepsilon})$ under the stated regularity and distribution-coverage assumptions. These are achievable-return relations between policy classes rather than guarantees for the optimized actor; the learned policy is evaluated empirically through the reconstruction-component ablations and aligned policy comparisons.

\paragraph{Explicit estimator interfaces vs. policy-internal representation shaping}

\begin{figure*}[thbp]
	\centering
	\includegraphics[width=0.9\textwidth]{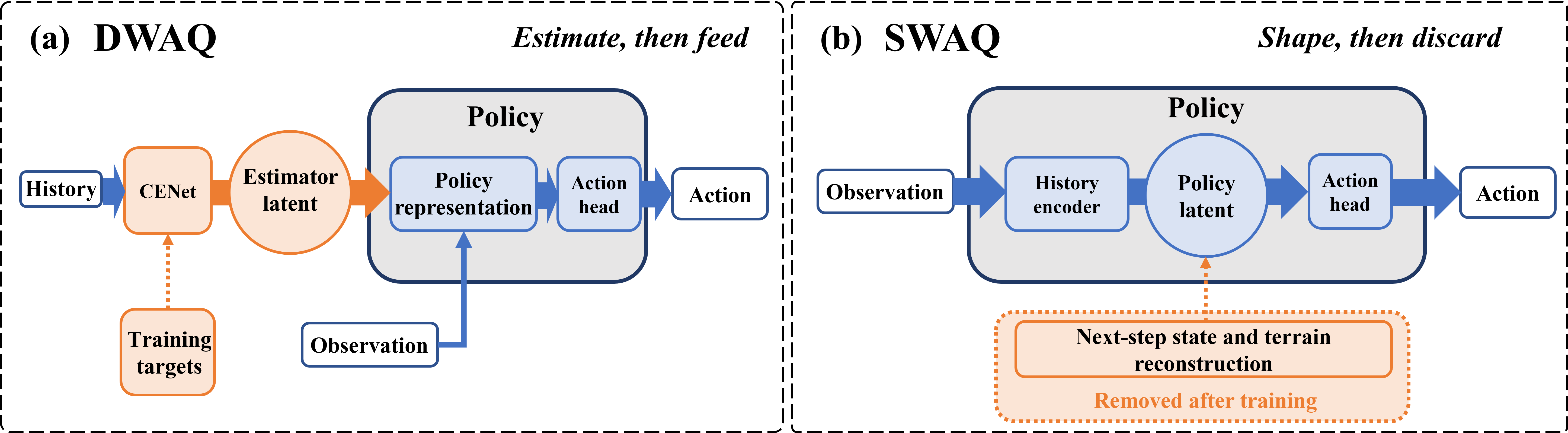}
	\caption{Comparison of information flow of DWAQ~\cite{aswinnahrendra2023} and SWAQ. (a) DWAQ uses a context-aided estimator to produce an estimated velocity and latent terrain context, which are explicitly supplied with the current observation to a downstream policy. (b) SWAQ applies a training-only reconstruction objective directly to the policy-internal history representation, while deployment retains only the history-to-action control pathway. Solid arrows denote deployed computation, whereas orange dotted arrows denote training-only supervision.}
	\label{fig:comparison_dwaq_swaq}
\end{figure*}

Fig.~\ref{fig:comparison_dwaq_swaq} reflects two different views of representation learning under partial observability. DWAQ treats the learned representation as an explicit interface between estimation and control: the estimator produces an intermediate representation for the downstream policy to consume. SWAQ instead treats semantic supervision as a means of shaping the policy's own decision representation. 
Its reconstruction objective specifies which robot-state and terrain information should remain recoverable from the policy-internal history representation, while leaving the policy free to organize this information for action generation. The key distinction is therefore one of representation ownership: DWAQ learns a representation for the policy, whereas SWAQ shapes the representation within the policy. In this sense, SWAQ moves semantic decomposition from the deployed computation graph into the training objective—it decomposes learning without decomposing inference. More generally, this suggests that information useful for resolving partial observability can serve as training-time supervision without becoming a mandatory runtime interface.

\subsubsection{Joint Representation-Policy Optimization}
\label{sec:joint_optimization}

SWAQ jointly updates the actor, critic, and reconstruction head using minibatches sampled from finite-horizon rollouts.
We denote the parameters of the recurrent history encoder and shared trunk, policy head, and critic by
$\theta_{\mathrm{enc}}$, $\theta_{\pi}$, and $\theta_{V}$, respectively. The parameters of the training-only state-and-terrain reconstruction head are denoted by $\theta_{\mathrm{rec}}$. The complete set of trainable parameters is given by
$
	\Theta
	=
	\left(
	\theta_{\mathrm{enc}},
	\theta_{\pi},
	\theta_V,
	\theta_{\mathrm{rec}}
	\right).
	\label{eq:swaq_parameters}
$

The policy and critic are trained using Proximal Policy Optimization (PPO)~\cite{schulman2017}. During each PPO update epoch, the PPO and state-and-terrain reconstruction objectives are evaluated on the same rollout minibatches and combined into the SWAQ training objective:
\begin{equation}
	\mathcal L_{\mathrm{SWAQ}}(\Theta)
	=
	\mathcal L_{\mathrm{PPO}}
	\left(
	\theta_{\mathrm{enc}},
	\theta_{\pi},
	\theta_V
	\right)
	+
	\lambda_{\mathrm{rec}}
	\mathcal L_{\mathrm{rec}}
	\left(
	\theta_{\mathrm{enc}},
	\theta_{\mathrm{rec}}
	\right).
	\label{eq:swaq_total_loss}
\end{equation}
where $\lambda_{\mathrm{rec}}$ control the contributions of the state-and-terrain reconstruction objective.

The PPO and reconstruction objectives are combined before backpropagation, and all trainable parameters are updated using a single optimizer. Denoting the learning rate by $\alpha$, the joint update can be written schematically as
\begin{equation}
	\Theta
	\leftarrow
	\Theta
	-
	\alpha
	\nabla_{\Theta}
	\mathcal L_{\mathrm{SWAQ}}(\Theta).
	\label{eq:swaq_parameter_update}
\end{equation}

Equations~\eqref{eq:swaq_total_loss} and \eqref{eq:swaq_parameter_update} make the gradient coupling explicit. The PPO objective updates the recurrent encoder, policy head, and critic, whereas the reconstruction objective updates the recurrent encoder and reconstruction head. Thus, $\theta_{\mathrm{enc}}$ is jointly shaped by policy learning and privileged physical reconstruction, while $\mathcal L_{\mathrm{rec}}$ does not directly update $\theta_{\pi}$ or $\theta_V$.

\subsection{Training and Implementation Details}

We train 4096 simulated robots in parallel. The terrain sampler uses ascending stairs (33.3\%), descending stairs (33.3\%), and equal shares of rough ground, uphill slopes, and downhill slopes (33.3\% in total). Stair treads range from $0.28$ to $0.35\,\mathrm{m}$, with step heights ranging from $0.0$ to $0.25\,\mathrm{m}$; slope magnitudes reach $0.35$. Terrain difficulty is increased progressively using curriculum learning~\cite{hwangbo2019}.
The physics simulation uses a $5\,\mathrm{ms}$ time step, and the policy gives a $20\,\mathrm{ms}$ control step ($50\,\mathrm{Hz}$). Each episode lasts $20\,\mathrm{s}$. 
Each training iteration collects 24 control steps from every environment, corresponding to 98,304 transitions in total. PPO then performs five update epochs with four minibatches per epoch. 

All policies, including ablation cases, are trained for 6000 iterations on a workstation with an Intel Core i9-14900K CPU, 128 GB RAM, and an NVIDIA RTX 4090 GPU.

The actor shown in Fig.~\ref{fig:Overview_of_SWAQ} uses a single-layer LSTM with 256 hidden units as the history encoder, a 256-unit MLP layer as the shared trunk, and two policy-head MLP layers with [256, 128] units. The auxiliary MLP decoder has two hidden layers with [128, 256] units. The critic uses a single-layer LSTM with 256 hidden units followed by three MLP layers with [512, 256, 128] units. All MLP layers use ELU activations.

\begin{table}[htbp]
	\centering
	\caption{Noise and Domain Randomization.}
	\begin{tabular}{lll}
		\toprule
		\textbf{Parameter Item} & \textbf{Value} & \textbf{Interval(s)} \\ 
		\hline
		\textbf{\textit{Material}}	&  &\\
		Friction & $\mathcal{U}(0.7, 1.2)$ & / \\
		Dynamic friction &  $\mathcal{U}(0.5, 0.8)$ & / \\
		Restitution &  $\mathcal{U}(0.0, 0.005)$ & / \\
		\textbf{\textit{Joint}} & & \\
		Friction  & $\exp(\mathcal{U}(\ln 0.1, \ln 10))\times$ default & / \\
		Armature & $\exp(\mathcal{U}(\ln 0.5, \ln 2))\times$ default & / \\
		Actuator stiffness & $\exp(\mathcal{U}(\ln 0.5, \ln 2))\times$ default & / \\
		Actuator damping & $\exp(\mathcal{U}(\ln 0.5, \ln 2))\times$ default & / \\
		\textbf{\textit{Link}}	&  & \\
		Mass/Inertia & $\mathcal{U}(0.8, 1.2)\times$ default & / \\
		\textbf{\textit{CoM Disturbance}} & & \\
		External force/torque & $\mathcal{U}(-20, 20)$ &  $\mathcal{U}(1, 5)$ \\
		Linear/angular velocity & $\mathcal{U}(-1, 1)$ / $\mathcal{U}(-0.5, 0.5)$ &  $\mathcal{U}(1, 5)$ \\
		\hline
	\end{tabular}
	\label{table:noise_and_domain_randomization}
\end{table}

\section{Simulation Experiments and Ablation Studies}

The experiments address four questions:
\begin{enumerate}
	\item[\textbf{Q1:}] What privileged physical and terrain semantics are recoverable from the learned history representations? We examine reconstructed terrain, autoregressive state rollouts, and the representation update triggered by the first foot-riser contact.
	
	\item[\textbf{Q2:}] Do these reconstructed semantics persist along the control pathway to the layers immediately upstream of the action output? We probe successive control-pathway representations.
	
	\item[\textbf{Q3:}] What theoretical support connects training-only privileged recoverability to the achievable control capability of a history-based policy class? Sections~\ref{app:information_retention_emulation} and~\ref{app:return_guarantees} of Appendix~\ref{app:proofs} connects reconstruction accuracy to information retention, privileged-policy emulation, and achievable-return bounds.
	
	\item[\textbf{Q4:}]  To what extent do the reconstruction components and the complete SWAQ design improve policy learning and locomotion performance? We use controlled ablations and aligned policy comparisons to quantify their effects.
\end{enumerate}

\subsection{Representation Semantics}

\subsubsection{Reconstructed State and Terrain Semantics}
We first decode the recurrent history representation into the robot state and local terrain quantities used by the auxiliary reconstruction objective. As shown in Fig.~\ref{fig:terrain_reconstruction_visualization}, the reconstructed height samples follow the overall terrain profile on level ground and during stair ascent and descent, indicating that the history representation retains information about nearby terrain geometry. The reconstruction is less accurate around stair edges and tends to replace abrupt height discontinuities with a smoother, nearly planar profile. This behavior is consistent with the averaging effect of pointwise squared-error regression~\cite{mathieu2016}; when the proprioceptive history does not uniquely determine a sharp terrain boundary, minimizing an $L_2$ reconstruction loss favors a conditional-mean estimate. Despite the loss of sharp geometric details, this smoothed terrain representation retains sufficient coarse information to support locomotion across the tested terrain types.

\begin{figure}[htbp]
	\centering
	\includegraphics[width=\columnwidth]{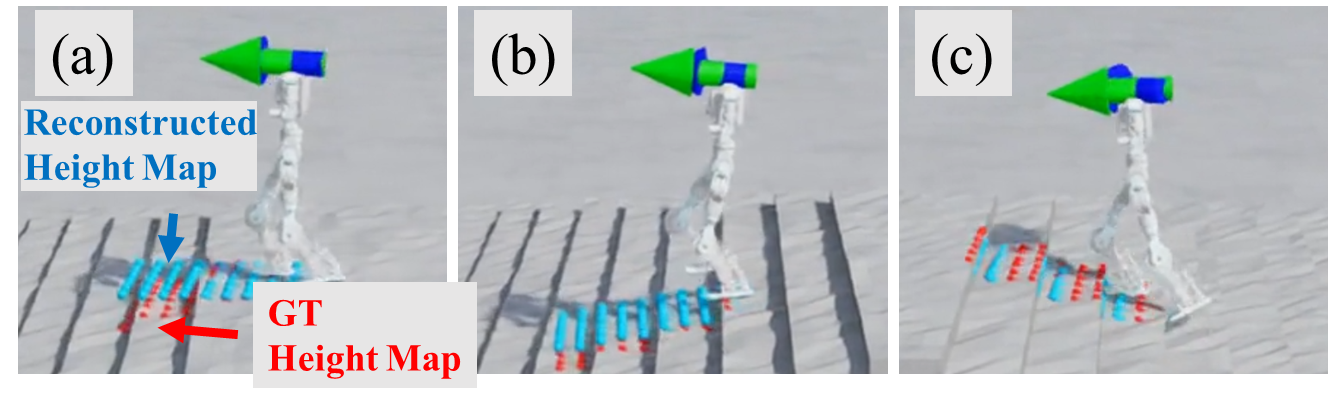}
	\caption{Training-only visualization of terrain reconstruction from the recurrent history representation: (a) Level Ground (ahead of the stairs), (b) Stair Descent, and (c) Stair Ascent. Red samples denote the simulator height-map target, whereas cyan samples denote its reconstruction.}
	\label{fig:terrain_reconstruction_visualization}
\end{figure}

\subsubsection{Autoregressive Future-state Visualization}

\begin{figure}[htbp]
	\centering
	\includegraphics[width=\columnwidth]{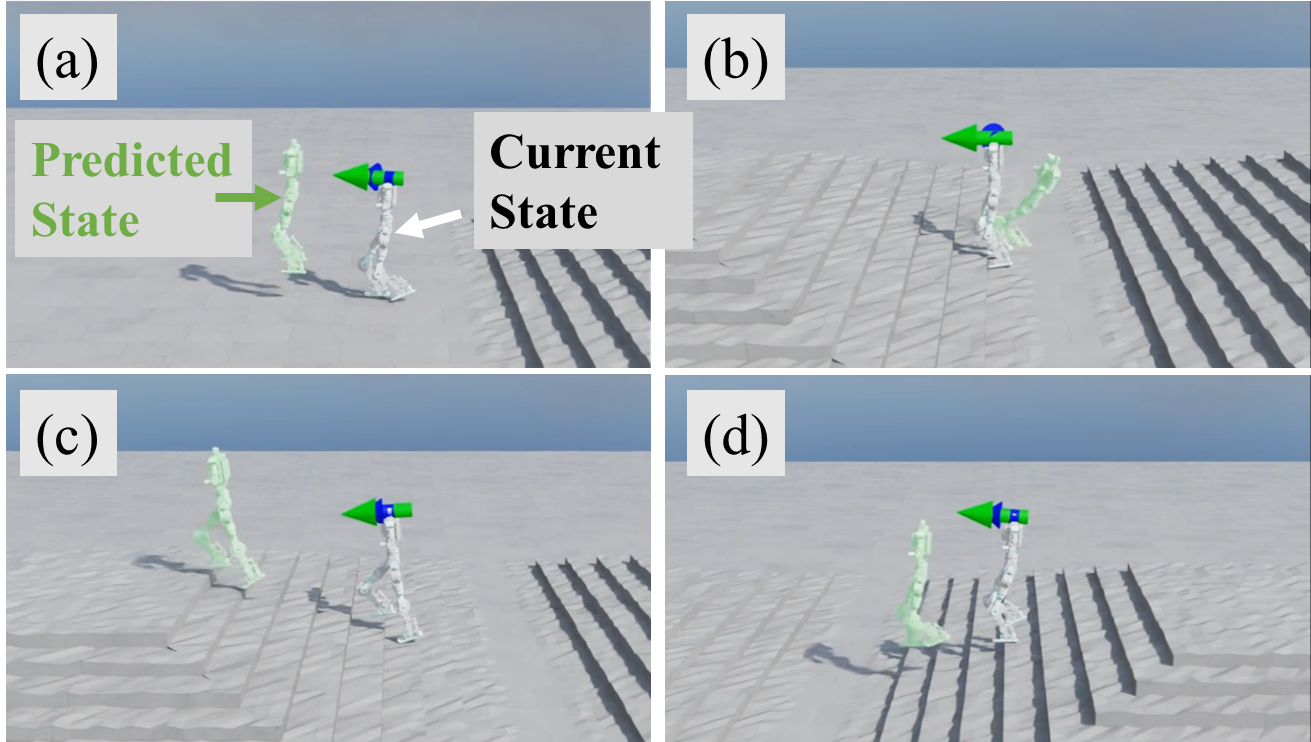}
	\caption{Training-only visualization of one-second decoder-actor rollouts: (a) level ground, (b) initial foot-riser contact, (c) steady stair ascent, and (d) stair descent. Each rollout starts from the recurrent state induced by the real history and then feeds the reconstructed observation back to the actor without ground-truth observation feedback. The opaque white robot shows the current state, whereas the translucent green robot shows the final reconstructed state.}
	\label{fig:future_state_visualization}
\end{figure}

We construct a one-second decoder-actor rollout to examine the temporal consistency of the reconstructed robot state. Starting from the recurrent state induced by the real observation history at time $t$, the shared reconstruction decoder produces the next-step robot state, clean policy observation, and terrain reconstruction. The reconstructed policy observation, including its reconstructed previous-action component, is then fed directly to the recurrent actor to update its memory and produce the next action and shared representation. Repeating this loop for 50 control steps without ground-truth observation feedback yields the rollout shown in Fig.~\ref{fig:future_state_visualization}. This diagnostic uses the reconstruction decoder $g_{\theta_{\mathrm{rec}}}$ for observation component while the reconstructed terrain component is not supplied to the actor.

The contrast between initial obstacle contact and steady stair ascent is particularly informative. Immediately after the foot strikes the first riser, recursively feeding the reconstructed observations produces a strongly destabilized pose. This visualization does not predict that the physical robot will fall, because reconstruction errors accumulate outside the one-step training distribution. During steady stair ascent, the rollout instead preserves the forward progression and alternating leg configuration associated with climbing. The difference indicates that the one-step reconstructions are more temporally self-consistent within the established climbing pattern than immediately after an unexpected collision.

Note that the rollout is used only as an offline diagnostic. Neither the reconstruction decoder nor the visualized future state participates in deployed action generation.

\subsubsection{Contact-Triggered Representation Update}

Fig.~\ref{fig:contact_event} treats the first foot-riser contact as a naturally occurring temporal boundary for examining how interaction-derived information enters the recurrent representation. We set $t_c=0$ at this event and align the robot motion, simulator contact forces, and reconstructed terrain quantities around it. The contact forces are logged only for analysis and are not included in the policy observation; the blind actor can sense the collision only indirectly through its effects on joint and body motion.

\begin{figure*}[htbp]
	\centering
	\includegraphics[width=0.8\textwidth]{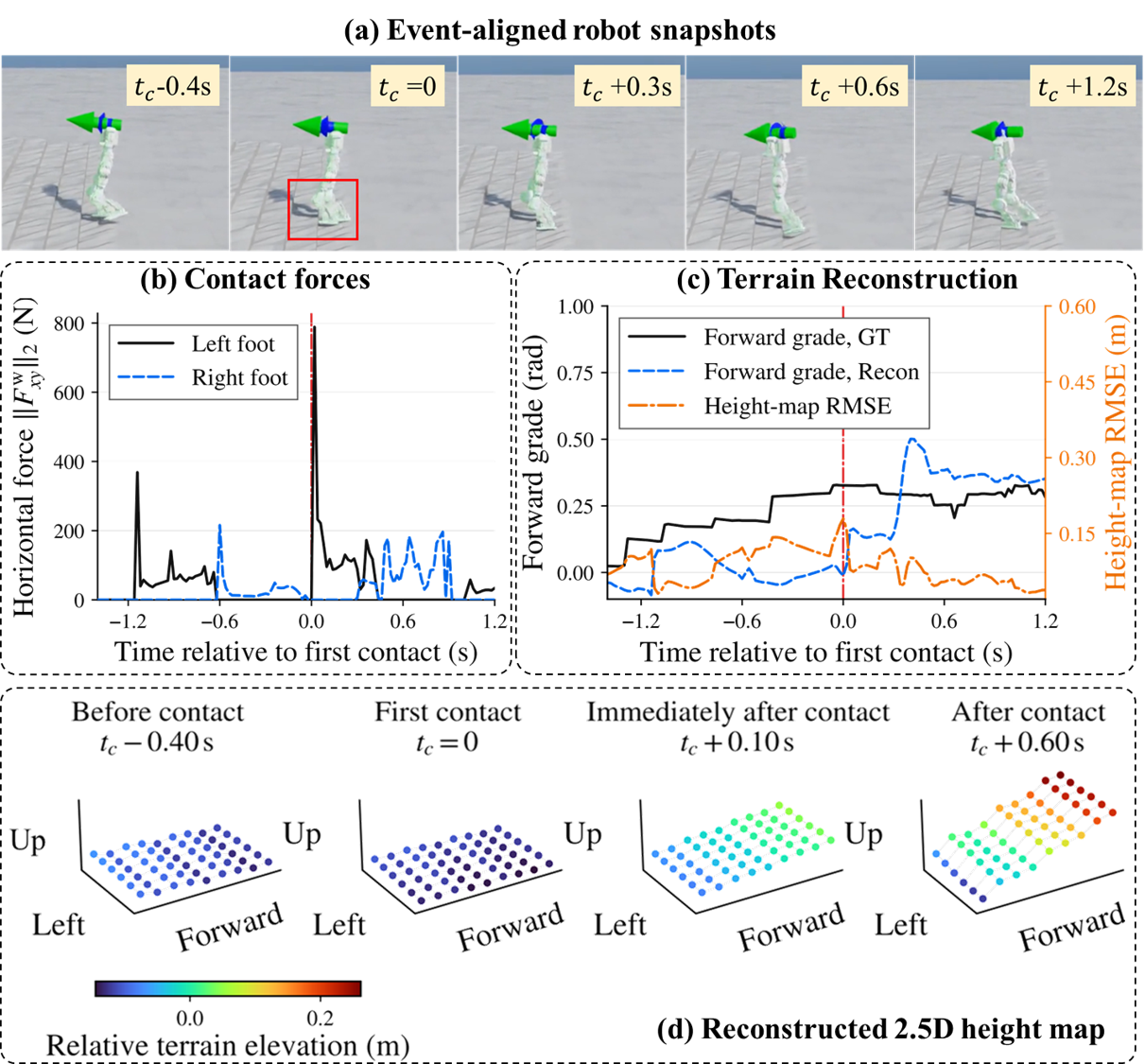}
	\caption{Contact-triggered analysis around the first foot-riser contact. (a) Time-aligned robot snapshots, with the first contact marked at $t_c=0$. (b) Horizontal contact-force magnitudes of the left and right feet; these simulator signals are used only for analysis. (c) Ground-truth and reconstructed forward terrain grades, together with the height-map RMSE. (d) Reconstructed 2.5-D height maps before and after contact. The initial collision is followed by a delayed and transiently exaggerated terrain reconstruction before the estimated grade approaches the ground truth.}
	\label{fig:contact_event}
\end{figure*}

At $t_c=0$, the horizontal contact force on the left foot rises abruptly to $0.79\,\mathrm{kN}$, identifying the first collision with the riser. The reconstructed forward grade simultaneously increases to $0.175\,\mathrm{rad}$ but remains below the ground-truth grade of about $0.3\,\mathrm{rad}$, while the height-map error reaches a local maximum. The reconstructed point grid therefore remains nearly flat at first contact: the collision provides new terrain evidence, but a single interaction does not determine the local geometry accurately.

The subsequent contacts refine this initial interpretation. From $t_c+0.3\,\mathrm{s}$, the right foot produces repeated horizontal-force peaks that are substantially smaller than the initial left-foot impact. During the same interval, the reconstructed grade continues to increase and transiently exceeds $0.5\,\mathrm{rad}$ at $t_c+0.42\,\mathrm{s}$, overshooting the ground-truth grade. By $t_c+0.52\,\mathrm{s}$, the reconstructed height map exhibits a coherent ascending structure. After the left foot establishes support on the first tread at $t_c+1.0\,\mathrm{s}$, the reconstructed grade moves toward the ground truth and the height-map error decreases.

This event-aligned sequence shows that terrain information is not recovered instantaneously at the first collision. Instead, the recurrent representation integrates the mechanical consequences of successive contacts, first forming an incomplete estimate, then revising it as the robot establishes a stable stair-ascent pattern. The first foot-riser contact thus acts as an interaction-induced information event that separates the unobserved approach from the subsequent accumulation of terrain evidence.

\subsection{Information Retention Along the Control Pathway}
\label{sec:information_retention}

\begin{figure*}[htbp]
	\centering
	\includegraphics[width=0.8\textwidth]{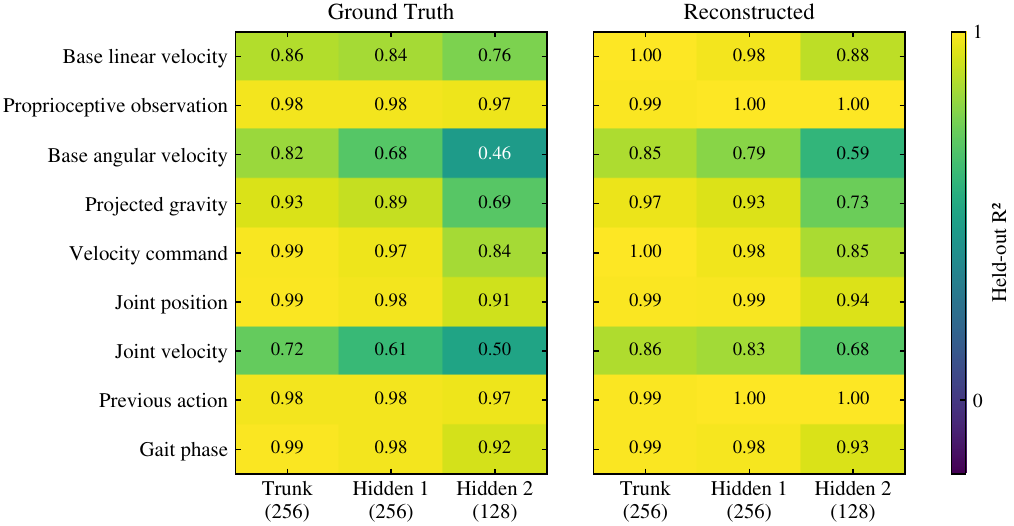}
	\caption{Held-out linear-probe $R^2$ for robot-related targets at successive control-pathway layers of the SWAQ policy trained for the humanoid robot. Ground truth and reconstructed report the layerwise decodability of their respective targets. Proprioceptive observation denotes the aggregate 50-dimensional policy-observation target.}
	\label{fig:probe_state}
\end{figure*}

\begin{figure*}[htbp]
	\centering
	\includegraphics[width=0.736\textwidth]{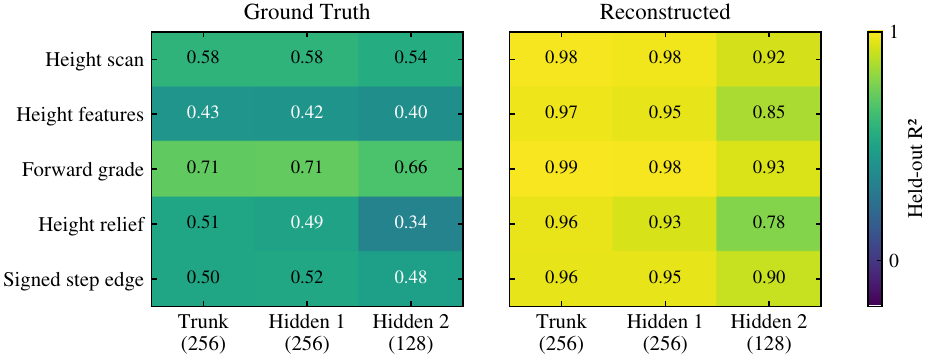}
	\caption{Held-out linear-probe $R^2$ for terrain-related targets at successive control-pathway layers. Height scan denotes the aggregate 54-dimensional terrain target.}
	\label{fig:probe_terrain}
\end{figure*}

We freeze the trained actor and fit independent ridge-regression probes
to the shared trunk (256-D), first policy-head hidden layer (256-D), and
second policy-head hidden layer (128-D). Hidden~2 is immediately upstream
of the action output. The probes use disjoint training and evaluation
environments. Their held-out performance is measured by
\begin{equation}
	R^2
	=
	1-
	\frac{\sum_i (y_i-\widehat y_i)^2}
	{\sum_i (y_i-\overline y_{\mathrm{test}})^2},
\end{equation}
where $R^2=1$ denotes exact prediction and $R^2=0$ matches prediction
by the test-set mean. For vector targets, the per-dimension scores are
aggregated using target-variance weighting. Ground truth and
reconstructed targets are probed separately; the comparison concerns
their layerwise decodability.

High decodability of reconstructed quantities at the shared trunk is expected because the auxiliary decoder is trained from this representation. The key result is their persistence through the two policy-head layers of the control pathway. These layers are optimized for action generation and receive no direct gradient from the training-only reconstruction pathway

Fig.~\ref{fig:probe_state} reports the robot-related probes. At the second hidden layer, the reconstructed quantities are more linearly decodable than their ground-truth counterparts. For example, $R^2$ increases from 0.76 to 0.88 for base linear velocity, from 0.46 to 0.59 for base angular velocity, and from 0.50 to 0.68 for joint velocity. The control-pathway representation is therefore more strongly aligned with the reconstructed state variables. The aggregate 50-dimensional policy-observation target remains almost fully decodable ($0.97/1.00$ for ground truth/reconstruction), while the previous action is the most strongly retained individual signal ($0.97/1.00$), consistent with action-conditioned recurrent control.

The terrain probes in Fig.~\ref{fig:probe_terrain} show a larger ground-truth reconstruction contrast. At the second hidden layer, the $R^2$ values for the aggregate height scan and height features increase from 0.54 and 0.40 for the current ground-truth terrain to 0.92 and 0.85 for their reconstructed counterparts. Forward grade is the most strongly retained scalar terrain descriptor, reaching 0.66 for the ground-truth target and 0.93 for its reconstruction. Because the blind policy receives no exteroceptive input, local terrain is inferred only after contact and body dynamics enter the proprioceptive history. The higher reconstructed-target $R^2$ is consistent with the recurrent representation consolidating this delayed evidence into its internal terrain estimate.

Across both target groups, the control-pathway representations are more linearly aligned with the internally reconstructed quantities, even when these quantities differ from the privileged ground truth. Their decodability through Hidden~2 shows that the corresponding information remains linearly accessible near the action output. The probe analysis characterizes representation-level persistence.
A complementary theoretical question is how imperfect recovery of privileged variables limits the return attainable by a history-based policy. Under the stated assumptions, Corollary~\ref{corollary:achievable_return_bound} bounds the achievable-return gap between the history-based and corresponding privileged-information policy classes in terms of the reconstruction error.

\subsection{Training Performance and Curriculum Progression}

\begin{figure*}[htbp]
	\centering
	\subfloat[Mean episode reward.]{\includegraphics[width=0.32\textwidth]{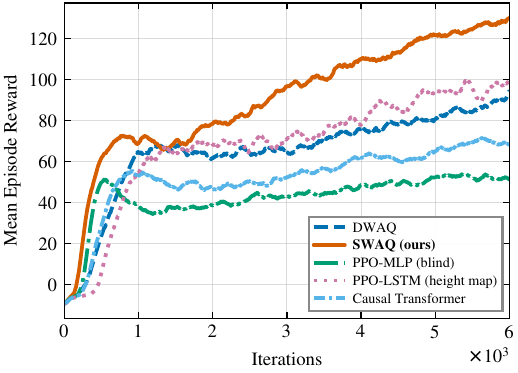}}
	\hfill
	\subfloat[Mean terrain level.]{\includegraphics[width=0.32\textwidth]{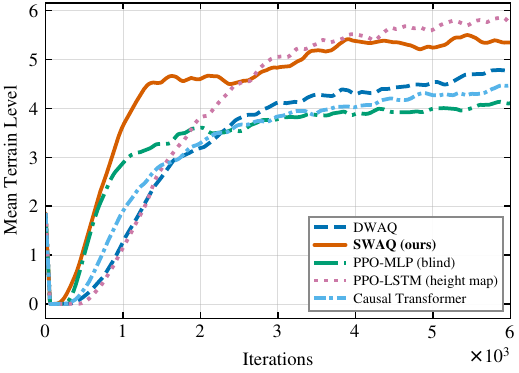}}
	\hfill
	\subfloat[Linear-velocity tracking reward.]{\includegraphics[width=0.32\textwidth]{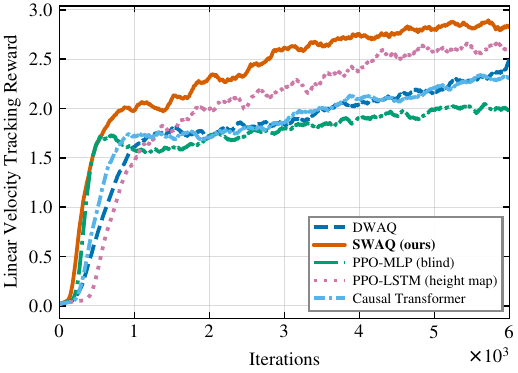}}
	\caption{Simulation training curves for SWAQ and the comparison policies. SWAQ attains the highest final return and a linear-velocity tracking
		reward close to that of the height-map policy while progressing to
		higher terrain levels than the non-exteroceptive comparison policies.}
	\label{fig:simulation_training_curves}
\end{figure*}

\begin{table*}[thbp]
	\centering
	\caption{Training outcomes and inference complexity.}
	\label{table:return_comparison}
	\renewcommand{\arraystretch}{1.08}
	\setlength{\tabcolsep}{4pt}
	\footnotesize
	\begin{tabular*}{\textwidth}{@{\extracolsep{\fill}}lcccccc@{}}
		\toprule
		\textbf{Policy}
		&
		\shortstack{\textbf{Training time}\\\textbf{(6000 iter.)}}
		&
		\shortstack{\textbf{Peak mean} \\\textbf{terrain level}}
		&
		\shortstack{\textbf{Mean $v_x$ error}\\\textbf{(last 100 iter.)}}
		&
		\shortstack{\textbf{Actor parameters}\\\textbf{(train/infer)}}
		&
		\shortstack{\textbf{Inference MACs}\\\textbf{per step}}
		&
		\shortstack{\textbf{Input}\\\textbf{dimension}}
		\\
		\midrule
		DWAQ
		& 2 h 56 min & 4.8168 & 0.4852
		& 865,164 & 0.861 M & $50+10\times50$ \\
		
		\textbf{SWAQ (Ours)}
		& 2 h 51 min & 5.5410 & 0.5668
		& 574,981/481,562 & 0.479 M & $50$ \\
		
		PPO-MLP (Blind)
		& 2 h 04 min & 4.1540 & 0.5867
		& 113,434 & 0.113 M & $50$ \\
		
		PPO-LSTM (Height Map)
		& 2 h 22 min & 5.8863 & 0.5570
		& 536,858 & 0.534 M & $50+54$ \\
		
		Causal Transformer
		& 8 h 26 min & 4.4888 & 0.5807
		& 424,378 & 6.583 M & $16\times50$ \\
		\bottomrule
	\end{tabular*}
	
	\vspace{2pt}
	\parbox{\textwidth}{\footnotesize
		The explicit history windows are expressed as
		context length $\times$ observation dimension.}
\end{table*}

Fig.~\ref{fig:simulation_training_curves} compares SWAQ with the estimator-conditioned DWAQ architecture~\cite{aswinnahrendra2023} and three task-objective-only baselines on the humanoid embodiment: blind PPO-MLP, height-map PPO-LSTM~\cite{rudin2022b}, and causal Transformer~\cite{radosavovic2024a}. Here, \emph{task-objective-only} means that the policy representation is learned solely through the PPO objective, without auxiliary reconstruction or estimation supervision. All methods use the same humanoid model, reward, environment, curriculum, and training settings, while their observations and representation architectures follow the corresponding baseline definitions. Two-stage blind-locomotion methods, such as \cite{lee2020,kumar2021a}, are not included. The three panels report the episode return, mean terrain level, and linear-velocity tracking reward.

SWAQ enters the high-return regime early and continues improving throughout training. Its final episode return exceeds those of all comparison policies. SWAQ also reaches higher terrain levels than the non-exteroceptive baselines while maintaining velocity-tracking performance close to the height-map PPO-LSTM. Although the height-map policy reaches the highest curriculum level, SWAQ narrows this gap without receiving exteroceptive terrain measurements.

With the reward and curriculum settings aligned, these differences cannot be attributed to task-specific reward shaping or a more favorable terrain schedule. They are consistent with SWAQ's auxiliary objectives shaping the shared history representation to retain terrain-relevant information from proprioceptive interactions and learn obstacle-negotiation behavior more effectively than DWAQ and other recurrent and causal learning baselines. 

One possible explanation is that privileged representation shaping provides an inductive bias about which robot--environment information should remain represented in the policy's recurrent state. Specifically, the selected reconstruction targets encode the designer's prior knowledge about which physical variables are likely to matter for locomotion, thereby reducing the burden of discovering this information solely from the task objective.
In DWAQ, semantic supervision is applied to a separate estimator interface, while the downstream policy must still learn from the task objective how to organize and use the estimated quantities. This representation-discovery burden is greater still for the causal Transformer, which must jointly discover which temporal evidence to preserve, propagate it through successive network layers, and map it into control actions. By applying semantic supervision directly to the policy-internal history representation, SWAQ may shorten this representation-discovery process. This interpretation is consistent with the matched reconstruction ablations and the observed curriculum progression, although the comparison does not isolate it from differences among the baseline architectures.

Table~\ref{table:return_comparison} summarizes the training outcomes and inference complexity. Despite having a comparable actor parameter count, the causal Transformer processes an explicit $16\times50$ history window at each control step and requires 6.583~M MACs, approximately $13.7\times$ the 0.479~M MACs required by SWAQ.
SWAQ instead consumes a 50-dimensional per-step input and carries historical information in its fixed-dimensional recurrent state.

\subsection{Effects of Reconstruction Components}
\begin{figure*}[htbp]
	\centering
	\subfloat[Mean episode reward.]{\includegraphics[width=0.32\textwidth]{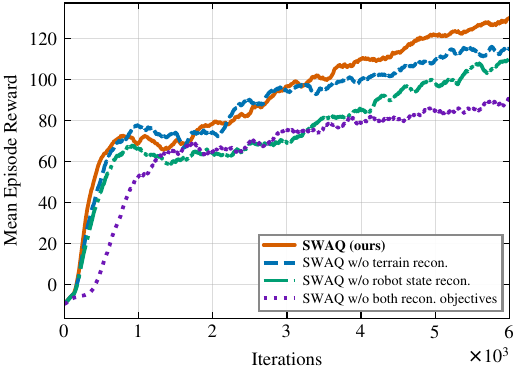}}\hfill\subfloat[Mean terrain level.]{\includegraphics[width=0.32\textwidth]{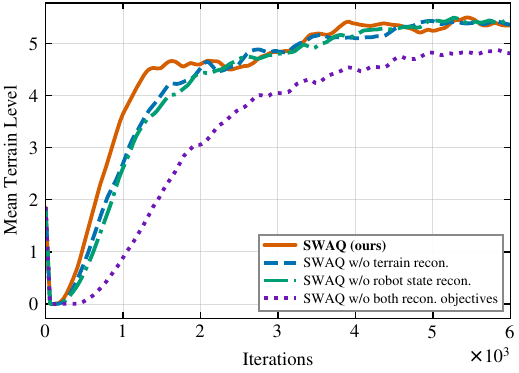}}\hfill\subfloat[Linear-velocity tracking reward.]{\includegraphics[width=0.32\textwidth]{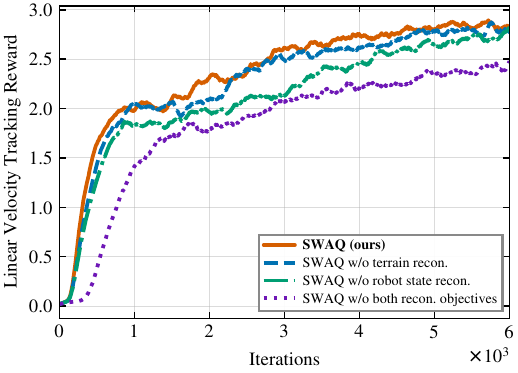}}\caption{Training curves for SWAQ and three reconstruction-objective ablations. Removing both robot-state and terrain reconstruction produces the largest and most persistent degradation in return, curriculum progression, and velocity tracking, whereas removing either component alone has a smaller effect.}
	\label{fig:auxiliary_objective_ablations}
\end{figure*}

To isolate the effect of reconstruction supervision, we retrain the same recurrent actor under identical PPO, reward, terrain, and curriculum settings. We remove terrain reconstruction and robot-state reconstruction individually (\emph{w/o terrain recons.} and \emph{w/o robot-state recons.}) and jointly (\emph{w/o both recons.}). The joint-removal variant corresponds to a reward-only blind PPO-LSTM policy and therefore provides the recurrent blind baseline complementary to the policy comparisons in Fig.~\ref{fig:simulation_training_curves}. All variants retain the same recurrent actor architecture and deployed control pathway.

As shown in Fig.~\ref{fig:auxiliary_objective_ablations}, jointly removing robot-state and terrain reconstruction produces the largest degradation across the three training curves, whereas either individual ablation retains more of the complete model's performance. The robot-state targets may also provide indirect terrain supervision because proprioceptive quantities such as leg joint angles encode terrain-dependent contact and kinematic responses. These results support using complementary reconstruction targets to shape the shared decision representation during training.

\section{Real-World Experiments}

We evaluate SWAQ on a humanoid and a quadruped robot (Go1) to test whether its representation-learning principle depends on a particular embodiment. Cross-platform evaluation has likewise been used to assess the scalability of legged-locomotion frameworks across robot morphologies~\cite{nahrendra2026}. The two SWAQ policies use embodiment-specific observation and action spaces but retain the same one-stage auxiliary-learning design and proprioception-only deployment setting.

Figs.~\ref{fig:humanoid_stair} and~\ref{fig:Go1_stair} show time-ordered snapshots of stair ascent and descent. Because the policies receive no exteroceptive terrain measurements, the first step cannot be identified before physical interaction. The red boxes therefore highlight the first foot-stair contact in each trial. After this contact, the policy produces a corrective leg-lifting motion and subsequently establishes a repeatable stair-negotiation gait without an explicit terrain classifier or a hand-designed control-mode switch.

\begin{figure*}[htbp]
	\centering
	\includegraphics[width=6.5 in]{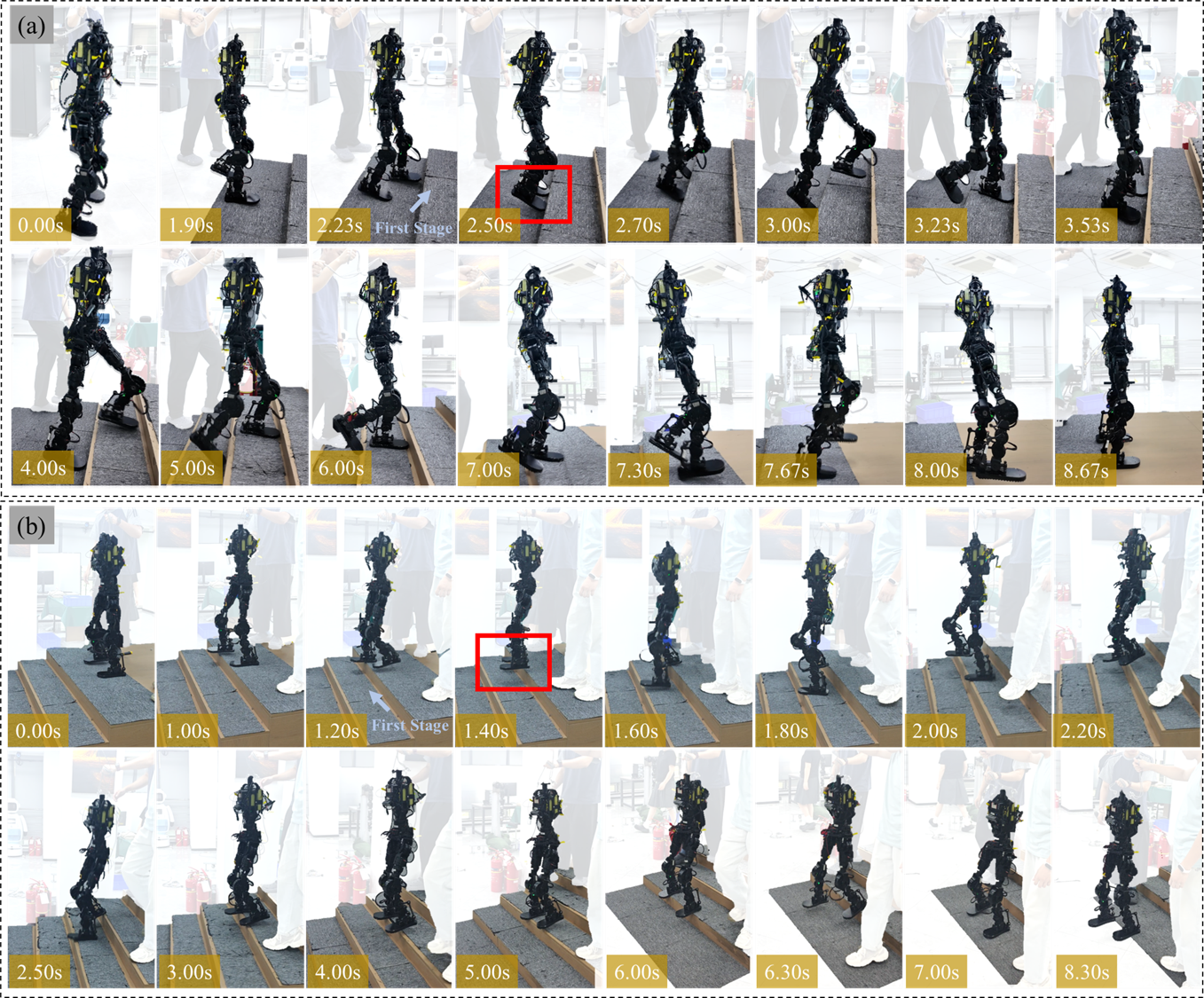}
	\caption{Time-ordered snapshots of real-world stair (with 6 stages, 30cm$\times$10cm of each stage) traversal by the humanoid robot: (a) ascent and (b) descent. The red boxes indicate the first foot-stair contact, and the white labels report the elapsed time from the beginning of each trial.}\label{fig:humanoid_stair}
\end{figure*}

For the humanoid robot, the first ascent contact occurs at $2.50\,\mathrm{s}$. The initially obstructed foot is lifted onto the first tread, after which alternating footholds carry the robot to the upper platform; the displayed sequence spans approximately $8.7\,\mathrm{s}$. During descent, the first interaction with the staircase occurs at $1.40\,\mathrm{s}$. The policy then regulates successive downward foot placements while maintaining balance, and the robot reaches the lower floor within the $8.3\,\mathrm{s}$ sequence shown in Fig.~\ref{fig:humanoid_stair}(b).

\begin{figure*}[htbp]
	\centering
	\includegraphics[width=6.5 in]{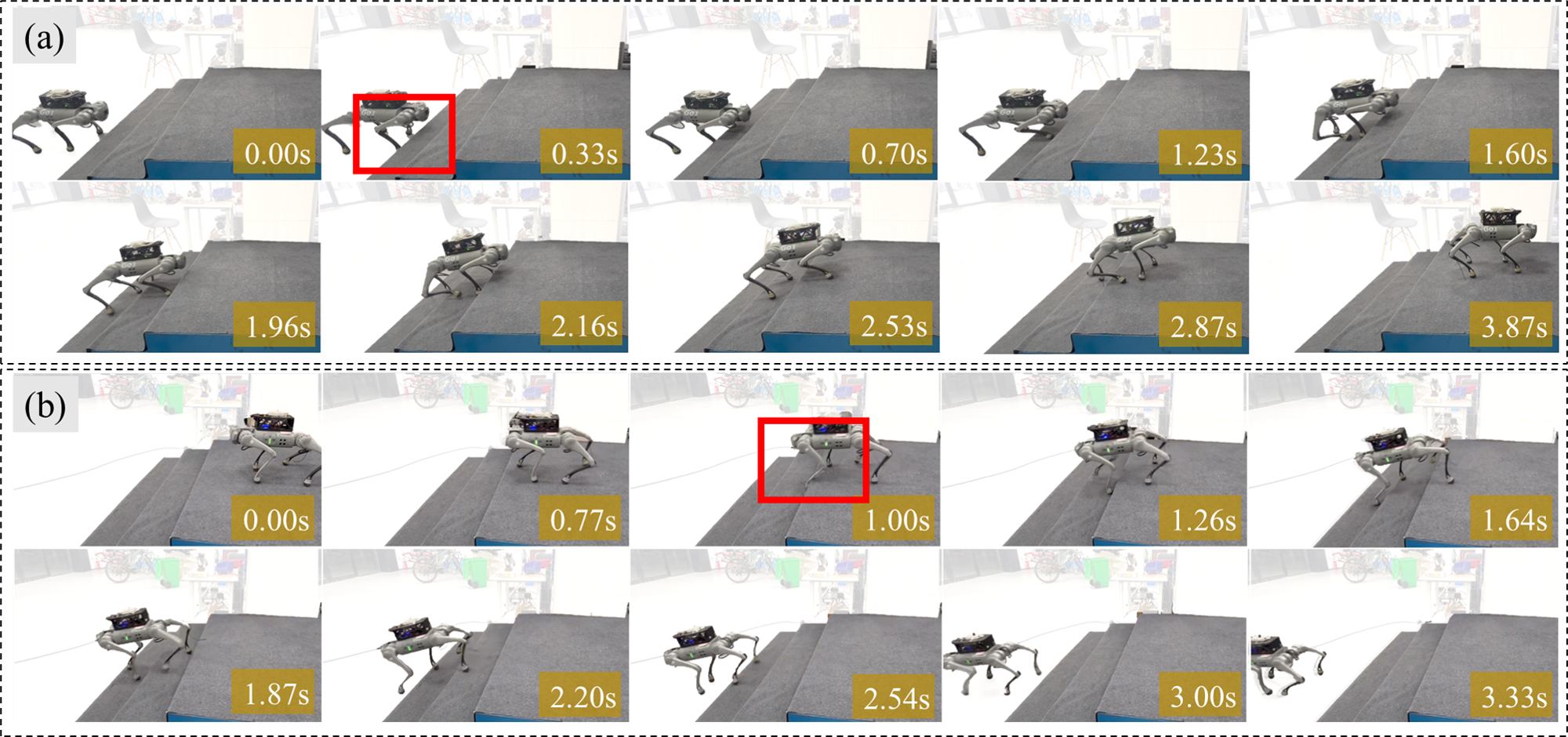}
	\caption{Time-ordered snapshots of real-world stair (with 3 stages, 30cm$\times$13cm of each stage) traversal by the quadruped robot (Unitree Go1): (a) ascent and (b) descent.}\label{fig:Go1_stair}
\end{figure*}

The Go1 provides a different leg number, contact sequence, and body geometry. In the ascent trial, the first foot-riser contact occurs at $0.33\,\mathrm{s}$; the contacting leg is subsequently raised, and the robot transitions into coordinated stair climbing before reaching the upper platform at approximately $3.87\,\mathrm{s}$. In the descent trial, the first staircase contact is highlighted at $1.00\,\mathrm{s}$. The policy then coordinates the fore and hind legs to traverse the remaining steps and reaches the lower floor at approximately $3.33\,\mathrm{s}$. Together with the humanoid results, these trials show that the proposed representation-learning principle supports interaction-driven adaptation across both bipedal and quadrupedal locomotion.

\section{Conclusion}
Focusing on blind locomotion in legged robots, this work frames partial observability as an information-retention problem and develops SWAQ as a concrete realization of this view. The main finding is that privileged physical supervision can shape what a recurrent history representation retains without requiring the corresponding variables to become explicit inputs to the deployed controller.

Layerwise analyses of SWAQ show that selected robot-dynamics and coarse terrain information remains linearly decodable up to the layer preceding the action output. Under aligned settings, SWAQ achieves higher return and terrain-curriculum progression than the non-exteroceptive baselines.
Successful humanoid and quadrupedal trials provide qualitative evidence that the same training design can produce interaction-driven stair traversal on both tested embodiments. Under the stated assumptions, the theoretical analysis establishes when auxiliary recoverability implies information retention and bounds the achievable-return loss under approximate reconstruction.

More broadly, SWAQ does not eliminate representation learning or prescribe how the supervised variables must be used. Instead, it brings part of the choice of what the representation should retain into the training objective, rather than leaving this choice to emerge only through task-return optimization. The policy remains free to organize its internal state and map it directly to actions. This provides an information-centric route to end-to-end policy design: semantic supervision can structure learning without imposing a corresponding architectural decomposition at deployment.

The current policy infers contact effects indirectly from proprioceptive motion. Future work will incorporate foot-force or tactile sensors to study more directly how contact information interacts with the learned history representation and robot motion. The training-only reconstruction pathway can also be extended to physical signals that are measurable in both simulation and real-world rollouts. Differences in their reconstruction errors or representations may provide a diagnostic view of the sim-to-real gap without making these signals part of the deployed control input. Whether such discrepancies can reliably diagnose this gap remains to be evaluated.
Finally, the present study considers auxiliary supervision for physical and terrain semantics within locomotion. More complex tasks may benefit from layer-specific auxiliary supervision. In robot soccer, for example, ball localization and goal-region segmentation could shape intermediate visual representations during training without becoming fixed inputs to a downstream controller. This would preserve a single deployable control pathway while allowing different auxiliary branches to supervise physical, geometric, and task-level semantics at different representation depths. Future work should determine when this design is preferable to explicit modular interfaces, which may remain valuable when calibrated outputs, interpretability, or independent verification are required.

\bibliographystyle{IEEEtran}

\bibliography{IEEEabrv,blind}

\appendices
\section{Proofs}
\label{app:proofs}

\subsection{Preliminaries and Notation}
\label{appendix:preliminaries}

\begin{assumption}[Temporal Alignment of the Reconstruction Target]
	\label{assumption:temporal_alignment}
	For a transition indexed by $t$, the auxiliary reconstruction pathway predicts the next-step target $\boldsymbol Y_{t+1}$ defined in \eqref{eq:reconstruction_target} from $\boldsymbol H_t$. Let $k=t+1$ denote the succeeding decision epoch. Equivalently,
	\begin{equation}
		\widehat{\boldsymbol Y}_k=g(\boldsymbol H_{k-1}).
		\label{eq:temporally_aligned_recovery}
	\end{equation}
	
	Here, $g$ denotes the parameter-suppressed theoretical counterpart of the implemented reconstruction mapping $g_{\theta_{\mathrm{rec}}}$.
	The target $\boldsymbol Y_k$ is formed from the simulator variables available at the beginning of decision epoch $k$, after the preceding transition and before $\boldsymbol A_k$ is selected. In particular, the action-history component of $\boldsymbol O_k^{\mathrm{clean}}$ is the previously executed action $\boldsymbol A_{k-1}$. At $k=0$, $\boldsymbol Y_0$, the previous-action entry, and the carried history representation $\boldsymbol H_{-1}$ are initialized by the simulator and actor reset conventions.
\end{assumption}

At decision epoch $k$, $\boldsymbol O_k$ denotes the current nonprivileged policy observation. It should not be confused with the privileged noise-free observation $\boldsymbol O_k^{\mathrm{clean}}$ contained in $\boldsymbol Y_k$. The recurrent actor first updates $\boldsymbol H_k=F(\boldsymbol H_{k-1},\boldsymbol O_k)$ and then generates the action from $\boldsymbol H_k$; hence, its complete decision rule is a measurable function of $(\boldsymbol O_k,\boldsymbol H_{k-1})$.

Let
\begin{equation*}
	X_k:=(s_k,\tau_k,\boldsymbol H_{k-1},\boldsymbol Y_k)\in\mathcal X
\end{equation*}
be the common augmented pre-action state used for analysis. It contains
the latent robot-environment state and the complete nonprivileged
history; the added representation and target components are updated
from the same underlying process. Hence, $(X_k)$ is Markov, and
\begin{equation*}
	(\boldsymbol O_k,\boldsymbol H_{k-1},\boldsymbol Y_k)=\Gamma(X_k).
\end{equation*}

To compare policies on this common process, we define the following induced policy classes on $\mathcal X$:
\begin{align*}
	\Pi_{\boldsymbol O}
	&:=\left\{x\mapsto\bar\pi(\cdot\mid \boldsymbol O(x))\right\},\\
	\Pi_{\boldsymbol O,\boldsymbol Y}
	&:=\left\{x\mapsto\bar\pi^+(\cdot\mid \boldsymbol O(x),\boldsymbol Y(x))\right\},\\
	\Pi_{\boldsymbol O,\boldsymbol H}
	&:=\left\{x\mapsto\widetilde\pi(\cdot\mid \boldsymbol O(x),\boldsymbol H(x))\right\},
\end{align*}
where the displayed decision rules range over admissible measurable mappings, and $\boldsymbol O(x)$, $\boldsymbol H(x)$, and $\boldsymbol Y(x)$ denote the coordinate projections corresponding to $\boldsymbol O_k$, $\boldsymbol H_{k-1}$, and $\boldsymbol Y_k$, respectively. We use $\mu$ for the population distribution of the temporally aligned tuple induced by the reconstruction-training data, rather than for an empirical minibatch distribution. Throughout the appendix, unindexed $(\boldsymbol O,\boldsymbol H,\boldsymbol Y)$ denotes a generic realization of $(\boldsymbol O_k,\boldsymbol H_{k-1},\boldsymbol Y_k)$.

\begin{lemma}[Baseline Information-Class Ordering]
	\label{lem:baseline_information_ordering}
	Let $\Pi_{\mathcal S}$ denote the class of admissible randomized
	Markov policies on the latent state $s_k$, and define
	$J_{\mathcal S}^{*}:=J^{*}(\Pi_{\mathcal S})$. Assume that the fully
	observed latent-state MDP admits a Markov policy that is optimal over
	all admissible nonanticipative policies. Then
	\begin{equation}
		J_{\mathcal S}^{*}
		\ge
		J_{\mathrm{hist}}^{*}
		\ge
		J_{\boldsymbol O,\boldsymbol H}^{*}
		\ge
		J_{\boldsymbol O}^{*}.
		\label{eq:baseline_information_ordering}
	\end{equation}
\end{lemma}

\begin{proof}
	Every policy in $\Pi_{\mathrm{hist}}$ is an admissible
	nonanticipative policy for the underlying latent-state control
	process. By the assumed sufficiency of Markov policies in the fully
	observed MDP, its return is no larger than $J_{\mathcal S}^{*}$.
	Taking the supremum over $\Pi_{\mathrm{hist}}$ gives
	$J_{\mathcal S}^{*}\ge J_{\mathrm{hist}}^{*}$.

	For a fixed recurrent encoder, $\boldsymbol O_k$ and
	$\boldsymbol H_{k-1}$ are measurable functions of $\tau_k$. Hence,
	for every $\widetilde\pi\in\Pi_{\boldsymbol O,\boldsymbol H}$, the
	decision rule
	\begin{equation*}
		\pi_{\mathrm{hist}}(\cdot\mid\tau_k)
		:=
		\widetilde\pi
		\bigl(
		\cdot\mid
		\boldsymbol O_k(\tau_k),
		\boldsymbol H_{k-1}(\tau_k)
		\bigr)
	\end{equation*}
	belongs to $\Pi_{\mathrm{hist}}$ and induces the same conditional
	action distributions and trajectory law. Taking suprema therefore
	gives $J_{\mathrm{hist}}^{*}\ge
	J_{\boldsymbol O,\boldsymbol H}^{*}$. Finally, every policy in
	$\Pi_{\boldsymbol O}$ is represented in
	$\Pi_{\boldsymbol O,\boldsymbol H}$ by ignoring
	$\boldsymbol H_{k-1}$, which gives
	$J_{\boldsymbol O,\boldsymbol H}^{*}\ge J_{\boldsymbol O}^{*}$.
\end{proof}

For any induced policy $\pi$, let $\rho_\pi^X$ denote its normalized discounted occupancy distribution on the common Markov state space $\mathcal X$. For every measurable $B\subseteq\mathcal X$,
\begin{equation*}
	\rho_\pi^X(B)
	:=
	(1-\gamma)
	\sum_{k=0}^{\infty}
	\gamma^k
	\Pr_\pi(X_k\in B).
\end{equation*}

\subsection{Information Retention and Exact Policy Emulation}
\label{app:information_retention_emulation}

\begin{proposition}[Reconstruction Accuracy Implies Representation-Level Information Retention]
	\label{prop:auxiliary_recovery_retention}
	
	Under Assumption~\ref{assumption:temporal_alignment}, let $(\boldsymbol H,\boldsymbol Y,\widehat{\boldsymbol Y})=(\boldsymbol H_{k-1},\boldsymbol Y_k,\widehat{\boldsymbol Y}_k)$ denote the temporally aligned variables, with
	\begin{equation*}
		\widehat{\boldsymbol Y}=g(\boldsymbol H).
	\end{equation*}
	
	Define the distortion
	\begin{equation*}
		d_{\mathcal Y}(\boldsymbol y,\widehat{\boldsymbol y})
		:=
		\left\|\boldsymbol W(\boldsymbol y-\widehat{\boldsymbol y})\right\|_2^2.
	\end{equation*}
	
	Suppose that the population reconstruction loss satisfies
	\begin{equation*}
		\mathcal L_{\mathrm{rec}}^\mu
		:=
		\mathbb E_\mu
		\left[d_{\mathcal Y}(\boldsymbol Y,\widehat{\boldsymbol Y})\right]
		\le\varepsilon.
	\end{equation*}
	
	Then
	\begin{equation}
		I_\mu(\boldsymbol Y;\boldsymbol H)
		\ge
		I_\mu(\boldsymbol Y;\widehat{\boldsymbol Y})
		\ge
		R_{\boldsymbol Y,\mu}(\varepsilon),
		\label{eq:representation_information_retention}
	\end{equation}
	where $R_{\boldsymbol Y,\mu}(\varepsilon)$ denotes the rate-distortion function of $\boldsymbol Y$ under distribution $\mu$ with respect to the distortion measure $d_{\mathcal Y}$.
\end{proposition} 
	
\begin{proof} 
	Since $\widehat{\boldsymbol Y}=g(\boldsymbol H)$ is a deterministic function of $\boldsymbol H$, the random variables form the Markov chain
	\begin{equation*}
		\boldsymbol Y\longrightarrow\boldsymbol H\longrightarrow\widehat{\boldsymbol Y}.
	\end{equation*}
	
	Therefore, by the data-processing inequality,
	\begin{equation}
		I_\mu(\boldsymbol Y;\boldsymbol H)
		\ge
		I_\mu(\boldsymbol Y;\widehat{\boldsymbol Y}).
		\label{eq:data_processing_reconstruction}
	\end{equation}
	
	By assumption,
	\begin{align*}
		\mathbb E_\mu\left[d_{\mathcal Y}(\boldsymbol Y,\widehat{\boldsymbol Y})\right]
		&=
		\mathbb E_\mu\left[\left\|\boldsymbol W(\boldsymbol Y-\widehat{\boldsymbol Y})\right\|_2^2\right]
		\\
		&\le\varepsilon.
	\end{align*}
	
	By the definition of the rate-distortion function, every reconstruction satisfying this distortion constraint obeys
	\begin{equation}
		I_\mu(\boldsymbol Y;\widehat{\boldsymbol Y})
		\ge
		R_{\boldsymbol Y,\mu}(\varepsilon).
		\label{eq:rate_distortion_lower_bound}
	\end{equation}
	
	Combining \eqref{eq:data_processing_reconstruction} and \eqref{eq:rate_distortion_lower_bound} yields
	\begin{equation*}
		I_\mu(\boldsymbol Y;\boldsymbol H)
		\ge
		I_\mu(\boldsymbol Y;\widehat{\boldsymbol Y})
		\ge
		R_{\boldsymbol Y,\mu}(\varepsilon).
	\end{equation*}
\end{proof}
	
\begin{remark}
	The result concerns the information contained in the representation as a whole. It does not imply that individual coordinates of $\boldsymbol H$ correspond to independent or disentangled physical quantities.
	
	The lower bound is nontrivial only when $\varepsilon$ is below the zero-information distortion level, i.e., below the distortion achievable by a constant predictor that ignores the representation.
\end{remark}
\quad

\begin{proposition}[Additional Privileged Information Cannot Decrease the Optimal Achievable Return]
	\label{prop:privileged_information_monotonicity}
	
	Let $\Pi_{\boldsymbol O}$ denote the class of policies conditioned only on the current nonprivileged observation $\boldsymbol O_k$, and let $\Pi_{\boldsymbol O,\boldsymbol Y}$ denote the class of policies conditioned on both $\boldsymbol O_k$ and the causally available privileged variable $\boldsymbol Y_k$. Define $J_{\boldsymbol O}^*:=\sup_{\pi\in\Pi_{\boldsymbol O}}J(\pi)$ and $J_{\boldsymbol O,\boldsymbol Y}^*:=\sup_{\pi^+\in\Pi_{\boldsymbol O,\boldsymbol Y}}J(\pi^+)$. 
	
	Then every policy in $\Pi_{\boldsymbol O}$ admits a behaviorally equivalent policy in $\Pi_{\boldsymbol O,\boldsymbol Y}$:
	\begin{equation*}
		\Pi_{\boldsymbol O}\subseteq\Pi_{\boldsymbol O,\boldsymbol Y},
	\end{equation*}
	and consequently, 
	\begin{equation}
		J_{\boldsymbol O,\boldsymbol Y}^*\ge J_{\boldsymbol O}^*.
		\label{eq:privileged_return_monotonicity}
	\end{equation}
\end{proposition}
\quad
\begin{proof} 
	For any policy $\pi\in\Pi_{\boldsymbol O}$, define a corresponding privileged-information policy $\pi^+\in\Pi_{\boldsymbol O,\boldsymbol Y}$ by 
	\[ \pi^+(\boldsymbol a\mid \boldsymbol o,\boldsymbol y):=\pi(\boldsymbol a\mid \boldsymbol o), \] 
	where the policy $\pi^+$ simply ignores the additional privileged variable $\boldsymbol y$.
	
	Therefore, $\pi$ and $\pi^+$ induce identical conditional action distributions at every decision step, and hence identical trajectory distributions and expected returns: 
	\[ J(\pi^+)=J(\pi). \] 
	
	Since every policy in $\Pi_{\boldsymbol O}$ can be represented in this form, we have 
	\[ \Pi_{\boldsymbol O}\subseteq\Pi_{\boldsymbol O,\boldsymbol Y}. \] 
	
	Taking the supremum over the larger policy class yields \[ J_{\boldsymbol O,\boldsymbol Y}^*\ge J_{\boldsymbol O}^*. \]
	\end{proof}
\begin{remark}
	This proposition establishes only a non-strict inequality. It guarantees that access to the privileged variable cannot decrease the optimal achievable return, but it does not imply that the privileged variable provides a strictly positive performance gain.
\end{remark}
	
\quad

\begin{lemma}[Exact Privileged-Variable Recovery Implies Privileged-Policy Emulation]
	\label{lem:exact_recovery_emulation}

	Assume that the privileged variable is exactly recoverable from the preceding shared representation, i.e., $\boldsymbol Y_k=g(\boldsymbol H_{k-1})$ for every admissible temporally aligned history. Assume further that, for every privileged-information policy $\pi^+\in\Pi_{\boldsymbol O,\boldsymbol Y}$, the hidden-state policy class $\Pi_{\boldsymbol O,\boldsymbol H}$ is sufficiently expressive to represent the composed policy
	\[ (\boldsymbol o,\boldsymbol h)\mapsto\pi^+(\cdot\mid \boldsymbol o,g(\boldsymbol h)). \]
	
	Then every privileged-information policy $\pi^+\in\Pi_{\boldsymbol O,\boldsymbol Y}$ can be exactly emulated by a hidden-state policy $\tilde{\pi}\in\Pi_{\boldsymbol O,\boldsymbol H}$. In particular,
	\begin{equation}
		\Pi_{\boldsymbol O,\boldsymbol Y}
		\subseteq
		\Pi_{\boldsymbol O,\boldsymbol H},
		\qquad
		J_{\boldsymbol O,\boldsymbol H}^*
		\ge
		J_{\boldsymbol O,\boldsymbol Y}^*.
		\label{eq:exact_recovery_emulation}
	\end{equation}
\end{lemma}
\quad
\begin{proof}
	For an arbitrary privileged-information policy $\pi^+\in\Pi_{\boldsymbol O,\boldsymbol Y}$, define the corresponding hidden-state policy as \[ \tilde{\pi}(\boldsymbol a\mid \boldsymbol o,\boldsymbol h) := \pi^+(\boldsymbol a\mid \boldsymbol o,g(\boldsymbol h)). \] 
	
	By the expressivity assumption, $\tilde{\pi}\in\Pi_{\boldsymbol O,\boldsymbol H}$. 
	
	Since $g(\boldsymbol H_{k-1})=\boldsymbol Y_k$, for every action $\boldsymbol a$ and every decision epoch $k$, 
	\[ \tilde{\pi}(\boldsymbol a\mid \boldsymbol O_k,\boldsymbol H_{k-1}) = \pi^+(\boldsymbol a\mid \boldsymbol O_k,\boldsymbol Y_k). \]
	
	Thus, the two policies induce identical conditional action distributions at every step. By induction over time, they induce identical trajectory distributions and therefore identical expected returns: \[ J(\tilde{\pi})=J(\pi^+). \]
	
	Since this construction applies to every $\pi^+\in\Pi_{\boldsymbol O,\boldsymbol Y}$, it follows that 
	\[ \Pi_{\boldsymbol O,\boldsymbol Y}\subseteq\Pi_{\boldsymbol O,\boldsymbol H}. \]
	
	Therefore, \[ J_{\boldsymbol O,\boldsymbol H}^* = \sup_{\pi\in\Pi_{\boldsymbol O,\boldsymbol H}}J(\pi) \ge \sup_{\pi^+\in\Pi_{\boldsymbol O,\boldsymbol Y}}J(\pi^+) = J_{\boldsymbol O,\boldsymbol Y}^*. \]
\end{proof}

\begin{remark}
	For a restricted neural-network policy class, an additional approximation error may arise if the composed policy is not represented exactly.
\end{remark}

\begin{corollary}[Exact Recovery Preserves the Nonprivileged Achievable Return]
	\label{corollary:exact_recovery_no_loss}
	Under the assumptions of Proposition~\ref{prop:privileged_information_monotonicity} and Lemma~\ref{lem:exact_recovery_emulation},
	\begin{equation}
		J_{\boldsymbol O,\boldsymbol H}^{*}
		\ge
		J_{\boldsymbol O,\boldsymbol Y}^{*}
		\ge
		J_{\boldsymbol O}^{*}.
		\label{eq:exact_recovery_return_order_appendix}
	\end{equation}
\end{corollary}

\begin{proof}
	The first inequality follows from Lemma~\ref{lem:exact_recovery_emulation}, and the second follows from Proposition~\ref{prop:privileged_information_monotonicity}.
\end{proof}
\quad

\subsection{Approximate Policy Emulation and Return Guarantees}
\label{app:return_guarantees}

Exact recovery is generally unavailable in learned systems. To extend the exact-emulation result to approximate recovery,
we introduce the following regularity assumptions.

\begin{assumption}[Bounded Reward]
	\label{assumption:Bounded_reward}
	The one-step reward is uniformly bounded, i.e.,
	\begin{equation*}
		|r(s,\boldsymbol a,s')|\le r_{\max},
		\qquad
		\forall (s,\boldsymbol a,s')\in\mathcal S\times\mathcal A\times\mathcal S,
	\end{equation*}
	for some finite constant $r_{\max}>0$. 
\end{assumption}

\begin{assumption}[Lipschitz Privileged-Information Policy] 
	\label{assumption:Lipschitz}
	
	Let $\pi^+\in\Pi_{\boldsymbol O,\boldsymbol Y}$ be a privileged-information policy. There exists a finite policy-dependent constant $L_{\pi^+}$ such that, for all policy observations $\boldsymbol o\in\mathcal O$ and all $\boldsymbol y,\boldsymbol y'\in\mathcal Y$,
	\begin{equation*}
		D_{\mathrm{TV}} \left( \pi^+(\cdot\mid \boldsymbol o,\boldsymbol y), \pi^+(\cdot\mid \boldsymbol o,\boldsymbol y') \right)
		\le L_{\pi^+}\|\boldsymbol W(\boldsymbol y-\boldsymbol y')\|_2.
	\end{equation*}
	The total variation distance~\cite{schulman2017a} between two continuous action distributions with densities $p$ and $q$ is defined as
	\begin{equation*}
		D_{\mathrm{TV}}(P,Q)
		:= \frac{1}{2} \int_{\mathcal A} |p(\boldsymbol a)-q(\boldsymbol a)|\,\mathrm d\boldsymbol a.
	\end{equation*}
\end{assumption}

\begin{assumption}
	[Training-Distribution Coverage]
	\label{assumption:Distributional_coverage}
	
	Let $\nu_{\pi^+}:=\Gamma_{\#}\rho_{\pi^+}^X$ denote the pushforward of the normalized discounted state-occupancy distribution onto the temporally aligned tuple $(\boldsymbol O,\boldsymbol H,\boldsymbol Y)$. Equivalently,
	for every measurable $B\subseteq\mathcal O\times\mathcal H\times\mathcal Y$,
	\begin{equation*}
		\nu_{\pi^+}(B)
		:=
		(1-\gamma)
		\sum_{k=0}^{\infty}
		\gamma^k
		\Pr_{\pi^+}
		\left(
		(\boldsymbol O_k,\boldsymbol H_{k-1},\boldsymbol Y_k)\in B
		\right).
	\end{equation*}
	
	The distribution $\nu_{\pi^+}$ is absolutely continuous with
	respect to $\mu$, with a uniformly bounded density ratio:
	\[
	\frac{\mathrm d\nu_{\pi^+}}{\mathrm d\mu}\le C_{\pi^+,\mu}
	\quad \mu\text{-almost everywhere},
	\]
	for some finite policy-dependent constant $C_{\pi^+,\mu}$.
	
\end{assumption}

\begin{theorem}[Approximate Reconstruction Implies a Bounded Privileged-Policy Emulation Error]
	\label{theorem:approximate_policy_emulation} Under Assumptions~\ref{assumption:temporal_alignment}--\ref{assumption:Distributional_coverage}, let $\pi^+\in\Pi_{\boldsymbol O,\boldsymbol Y}$ be a privileged-information policy with constants $L_{\pi^+}$ and $C_{\pi^+,\mu}$. Suppose that \[ \mathcal L_{\mathrm{rec}}^\mu\le\varepsilon. \] 
	
	Define the reconstruction-induced hidden-state policy at decision epoch $k$ as \[ \tilde{\pi}_g(\boldsymbol a_k\mid \boldsymbol o_k,\boldsymbol h_{k-1}) := \pi^+(\boldsymbol a_k\mid \boldsymbol o_k,g(\boldsymbol h_{k-1})), \] where $g(\boldsymbol h_{k-1})$ reconstructs the privileged variable $\boldsymbol Y_k$ from the preceding hidden representation. Suppose that the composed policy belongs to the hidden-state policy class, i.e., $\tilde{\pi}_g\in\Pi_{\boldsymbol O,\boldsymbol H}$. Then 
	\begin{equation}
		\left| J(\pi^+)-J(\tilde{\pi}_g) \right| \le \frac{2r_{\max}L_{\pi^+}}{(1-\gamma)^2} \sqrt{C_{\pi^+,\mu}\varepsilon}.
		\label{eq:approximate_policy_emulation_bound}
	\end{equation} 
\end{theorem}

\begin{proof} 
	By Assumption~\ref{assumption:temporal_alignment}, $\boldsymbol O_k$, $\boldsymbol H_{k-1}$, and $\boldsymbol Y_k$ are functions of the common Markov state $X_k$, and unindexed $(\boldsymbol O,\boldsymbol H,\boldsymbol Y)$ below denotes this aligned tuple. We first transfer the reconstruction guarantee from the training distribution $\mu$ to the aligned-tuple occupancy distribution $\nu_{\pi^+}$. By Assumption~\ref{assumption:Distributional_coverage}, $\nu_{\pi^+}$ is absolutely continuous with respect to $\mu$, and \[ \frac{\mathrm d\nu_{\pi^+}}{\mathrm d\mu} \le C_{\pi^+,\mu} \] $\mu$-almost everywhere. Therefore, the expected squared recovery error under $\nu_{\pi^+}$ is bounded by \begin{align} &\mathbb E_{\nu_{\pi^+}} \left[ \|\boldsymbol W(\boldsymbol Y-g(\boldsymbol H))\|_2^2 \right] \nonumber\\ &\quad = \int \|\boldsymbol W(\boldsymbol y-g(\boldsymbol h))\|_2^2 \frac{\mathrm d\nu_{\pi^+}}{\mathrm d\mu} \,\mathrm d\mu \nonumber\\ &\quad \le C_{\pi^+,\mu} \int \|\boldsymbol W(\boldsymbol y-g(\boldsymbol h))\|_2^2 \,\mathrm d\mu \nonumber\\ &\quad = C_{\pi^+,\mu}\mathcal L_{\mathrm{rec}}^\mu \le C_{\pi^+,\mu}\varepsilon. \label{eq:occupancy_recovery_bound} \end{align} 
	
	Applying the Cauchy-Schwarz inequality to \eqref{eq:occupancy_recovery_bound}, the corresponding expected recovery error satisfies \begin{align} \mathbb E_{\nu_{\pi^+}} \left[ \|\boldsymbol W(\boldsymbol Y-g(\boldsymbol H))\|_2 \right] &\le \sqrt{ \mathbb E_{\nu_{\pi^+}} \left[ \|\boldsymbol W(\boldsymbol Y-g(\boldsymbol H))\|_2^2 \right] } \nonumber\\ &\le \sqrt{C_{\pi^+,\mu}\varepsilon}. \label{eq:expected_recovery_bound} \end{align} 
	
	We next translate the privileged-variable recovery error into a difference between the two action distributions. By definition, 
	\[ 
	\tilde{\pi}_g(\cdot\mid \boldsymbol O,\boldsymbol H) = \pi^+(\cdot\mid \boldsymbol O,g(\boldsymbol H)). 
	\] 
	
	Using the Lipschitz condition in Assumption~\ref{assumption:Lipschitz}, we obtain \begin{align*} &D_{\mathrm{TV}} \left( \pi^+(\cdot\mid \boldsymbol O,\boldsymbol Y), \tilde{\pi}_g(\cdot\mid \boldsymbol O,\boldsymbol H) \right) \\ &\quad = D_{\mathrm{TV}} \left( \pi^+(\cdot\mid \boldsymbol O,\boldsymbol Y), \pi^+(\cdot\mid \boldsymbol O,g(\boldsymbol H)) \right) \\ &\quad \le L_{\pi^+}\|\boldsymbol W(\boldsymbol Y-g(\boldsymbol H))\|_2. \end{align*} 
	
	Because $\nu_{\pi^+}=\Gamma_{\#}\rho_{\pi^+}^X$, taking expectations and applying \eqref{eq:expected_recovery_bound} gives
	\begin{align}
		&\mathbb E_{\rho_{\pi^+}^X}\!\left[
		D_{\mathrm{TV}}\!\left(
		\pi^+(\cdot\mid X),
		\tilde{\pi}_g(\cdot\mid X)
		\right)\right] \nonumber\\
		&\quad =
		\mathbb E_{\nu_{\pi^+}}\!\left[
		D_{\mathrm{TV}}\!\left(
		\pi^+(\cdot\mid \boldsymbol O,\boldsymbol Y),
		\tilde{\pi}_g(\cdot\mid \boldsymbol O,\boldsymbol H)
		\right)\right] \nonumber\\
		&\quad \le
		L_{\pi^+}\sqrt{C_{\pi^+,\mu}\varepsilon}.
		\label{eq:expected_policy_tv_bound}
	\end{align}
	
	It remains to translate the difference between the action distributions into a return difference. Both $\pi^+$ and $\tilde{\pi}_g$ are induced policies on the same history-induced MDP: \[ \pi^+(\cdot\mid X) = \pi^+(\cdot\mid \boldsymbol O(X),\boldsymbol Y(X)), \] and \[ \tilde{\pi}_g(\cdot\mid X) = \pi^+(\cdot\mid \boldsymbol O(X),g(\boldsymbol H(X))). \] 
	
	Define their action-value difference at $X$ as
	\begin{align*}
		\Delta_Q(X)
		&:= \mathbb E_{\boldsymbol a\sim\pi^+(\cdot\mid X)}
		\left[Q^{\tilde{\pi}_g}(X,\boldsymbol a)\right] \\
		&\quad - \mathbb E_{\boldsymbol a\sim\tilde{\pi}_g(\cdot\mid X)}
		\left[Q^{\tilde{\pi}_g}(X,\boldsymbol a)\right].
	\end{align*}
	
	Applying the performance-difference lemma~\cite{Kakade2002,schulman2017a} with $\pi'=\pi^+$ and $\pi=\tilde{\pi}_g$ gives
	\begin{equation}
		J(\pi^+)-J(\tilde{\pi}_g)
		= \frac{1}{1-\gamma}
		\mathbb E_{\rho_{\pi^+}^X}\!\left[\Delta_Q(X)\right].
		\label{eq:performance_difference_q}
	\end{equation}
	
	Here, the advantage function in the standard performance-difference lemma has been expanded using \[ A^{\tilde{\pi}_g}(X,\boldsymbol a) = Q^{\tilde{\pi}_g}(X,\boldsymbol a) - V^{\tilde{\pi}_g}(X), \] together with \[ V^{\tilde{\pi}_g}(X) = \mathbb E_{\boldsymbol a\sim\tilde{\pi}_g(\cdot\mid X)} \left[ Q^{\tilde{\pi}_g}(X,\boldsymbol a) \right]. \] 
	
	For any bounded function $f$ and any probability distributions $P$ and $Q$, 
	\[ \left| \mathbb E_P[f]-\mathbb E_Q[f] \right| \le 2\sup_{\boldsymbol a}|f(\boldsymbol a)|D_{\mathrm{TV}}(P,Q). \] 
	
	In the present case, we take $f(\boldsymbol a)=Q^{\tilde{\pi}_g}(X,\boldsymbol a)$. 
	
	By Assumption~\ref{assumption:Bounded_reward}, \begin{align} \left| Q^{\tilde{\pi}_g}(X,\boldsymbol a) \right| &= \left| \mathbb E_{\tilde{\pi}_g} \left[ \sum_{\ell=0}^{\infty} \gamma^\ell r_{k+\ell} \;\middle|\; X_k=X,\ \boldsymbol a_k=\boldsymbol a \right] \right| \nonumber\\ &\le \mathbb E_{\tilde{\pi}_g} \left[ \sum_{\ell=0}^{\infty} \gamma^\ell|r_{k+\ell}| \;\middle|\; X_k=X,\ \boldsymbol a_k=\boldsymbol a \right] \nonumber\\ &\le \sum_{\ell=0}^{\infty} \gamma^\ell r_{\max} = \frac{r_{\max}}{1-\gamma}. \label{eq:q_function_bound} \end{align} 
	
	Combining the total-variation inequality, \eqref{eq:performance_difference_q}, and \eqref{eq:q_function_bound}, we obtain
	\begin{align*}
		\left| J(\pi^+)-J(\tilde{\pi}_g) \right|
		&\le \frac{2r_{\max}}{(1-\gamma)^2}
		\mathbb E_{\rho_{\pi^+}^X} \Big[
		D_{\mathrm{TV}} \big(
		\pi^+(\cdot\mid X), \\
		&\hspace{34mm}
		\tilde{\pi}_g(\cdot\mid X)
		\big)\Big] \\
		&\le \frac{2r_{\max}L_{\pi^+}}{(1-\gamma)^2}
		\sqrt{C_{\pi^+,\mu}\varepsilon},
	\end{align*}
	where the final inequality follows by substituting the expected total-variation bound in \eqref{eq:expected_policy_tv_bound}. This proves the claimed privileged-policy emulation bound. 
\end{proof}

\quad

For a fixed reconstruction mapping $g$ and finite constants $L$ and $C$, let $\Pi_{\boldsymbol O,\boldsymbol Y}^{L,C}\subseteq\Pi_{\boldsymbol O,\boldsymbol Y}$ denote the class of privileged-information policies for which Assumption~\ref{assumption:Lipschitz} holds with $L_{\pi^+}\le L$, Assumption~\ref{assumption:Distributional_coverage} holds with $C_{\pi^+,\mu}\le C$, and the reconstruction-induced policy $\tilde{\pi}_g$ belongs to $\Pi_{\boldsymbol O,\boldsymbol H}$.

\begin{corollary}[Achievable-Return Bound under Approximate Privileged-Variable Recovery]
	\label{corollary:achievable_return_bound} 
	Define \[ J_{\boldsymbol O,\boldsymbol Y}^{*,L,C} := \sup_{\pi^+\in\Pi_{\boldsymbol O,\boldsymbol Y}^{L,C}}J(\pi^+), \quad J_{\boldsymbol O,\boldsymbol H}^{*} := \sup_{\pi\in\Pi_{\boldsymbol O,\boldsymbol H}}J(\pi). \] 
	
	If $\mathcal L_{\mathrm{rec}}^\mu\le\varepsilon$, then
	\begin{equation}
		J_{\boldsymbol O,\boldsymbol Y}^{*,L,C}-J_{\boldsymbol O,\boldsymbol H}^{*} \le \frac{2r_{\max}L}{(1-\gamma)^2} \sqrt{C\varepsilon}.
		\label{eq:restricted_class_return_bound}
	\end{equation}
\end{corollary} 

\begin{proof} 
	For any $\delta>0$, choose a policy $\pi_\delta^+\in\Pi_{\boldsymbol O,\boldsymbol Y}^{L,C}$ satisfying \[ J(\pi_\delta^+) \ge J_{\boldsymbol O,\boldsymbol Y}^{*,L,C}-\delta. \] 
	
	By the definition of $\Pi_{\boldsymbol O,\boldsymbol Y}^{L,C}$ and Theorem~\ref{theorem:approximate_policy_emulation}, there exists a corresponding hidden-state policy $\tilde{\pi}_{\delta,g}\in\Pi_{\boldsymbol O,\boldsymbol H}$ such that \[ J(\pi_\delta^+)-J(\tilde{\pi}_{\delta,g}) \le \frac{2r_{\max}L}{(1-\gamma)^2} \sqrt{C\varepsilon}. \] 
	
	Since $J_{\boldsymbol O,\boldsymbol H}^*$ is the supremum over all policies in $\Pi_{\boldsymbol O,\boldsymbol H}$, \[ J_{\boldsymbol O,\boldsymbol H}^* \ge J(\tilde{\pi}_{\delta,g}). \] 
	
	Therefore, \[ J_{\boldsymbol O,\boldsymbol Y}^{*,L,C}-J_{\boldsymbol O,\boldsymbol H}^{*} \le \delta + \frac{2r_{\max}L}{(1-\gamma)^2} \sqrt{C\varepsilon}. \] 
	
	Letting $\delta\to0$ proves the claimed inequality.
\end{proof}

\quad

\onecolumn
\section{Supplementary Experiments}

\FloatBarrier
\subsection{Supplementary Quadruped Representation Probes}

We apply the frozen-actor ridge-probe protocol described in Section~\ref{sec:information_retention} to the quadruped robot to test whether the humanoid representation-retention pattern persists under different body geometry and contact dynamics.

\begin{figure}[H]
	\centering
	\includegraphics[width=0.92\textwidth]{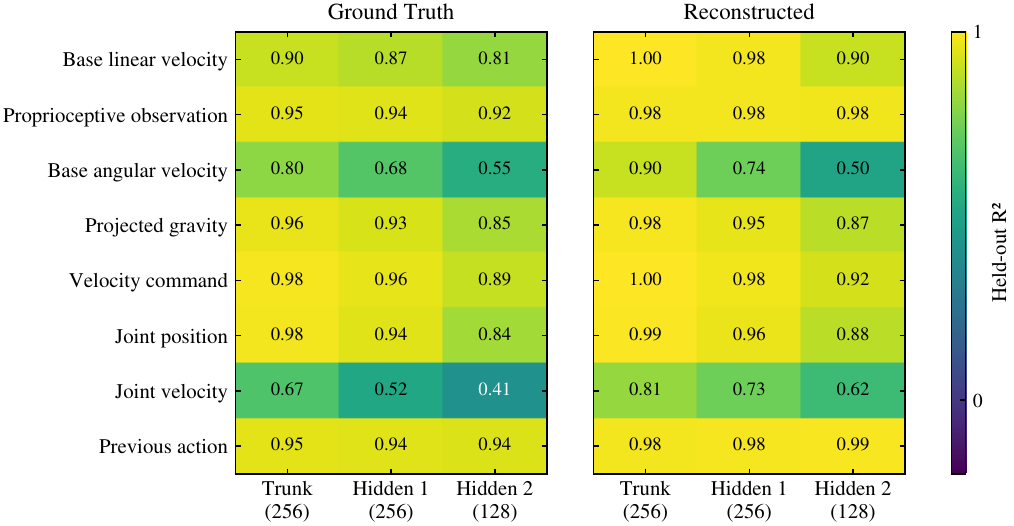}
	\caption{Held-out linear-probe $R^2$ for robot-related targets on the Unitree Go1 at successive control-pathway layers. Ground Truth and Reconstructed report the layerwise decodability of their respective targets. The aggregate proprioceptive observation and previous action remain strongly decodable through Hidden~2.}
	\label{fig:appendix_go1_state_probe}
\end{figure}

\begin{figure}[H]
	\centering
	\includegraphics[width=0.92\textwidth]{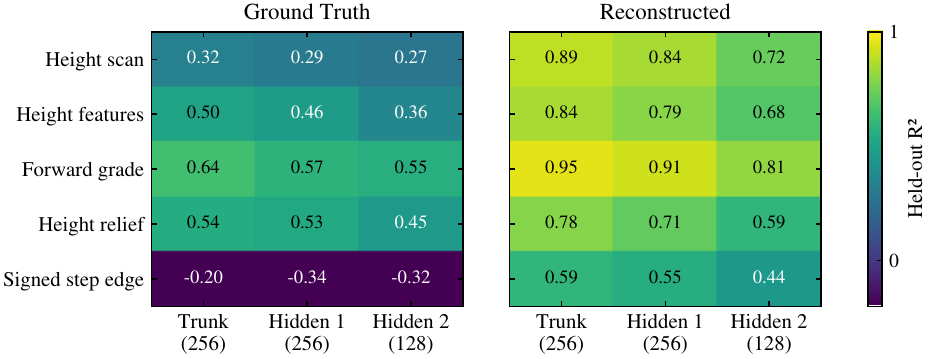}
	\caption{Held-out linear-probe $R^2$ for terrain-related targets on the Unitree Go1 at successive control-pathway layers. Reconstructed terrain quantities are more decodable than their ground-truth counterparts, with forward grade retained most strongly. Negative $R^2$ for the ground-truth signed step edge means that the linear probe performs worse than prediction by the test-set mean.}
	\label{fig:appendix_go1_terrain_probe}
\end{figure}

The Go1 probes broadly reproduce the humanoid pattern: reconstructed variables remain decodable near the action output. At Hidden~2, ground-truth projected gravity reaches $0.85$ on the Go1 versus $0.69$ on the humanoid, whereas reconstructed height-scan decodability is lower ($0.72$ versus $0.92$). Together with forward grade remaining the most decodable scalar terrain descriptor, this contrast is consistent with the Go1 organizing terrain evidence more around body attitude and the humanoid retaining a more explicit local-geometry representation. In both embodiments, reconstructed quantities persist through Hidden~2, supporting the conclusion that the training-only objective shapes information retained along the control pathway.

\FloatBarrier

\vfill

\end{document}